\documentclass[letterpaper, 10 pt, conference]{ieeeconf}  

\IEEEoverridecommandlockouts                              
\usepackage{graphicx} 
\usepackage{epsfig} 
\usepackage{mathptmx} 
\usepackage{times} 
\usepackage{amsmath} 
\usepackage{amssymb}  
\usepackage{hyperref}
\usepackage{caption}
\usepackage{subcaption}
\usepackage{grffile}
\usepackage{cite}
\usepackage{lipsum}                     
\usepackage{xargs}                      
\usepackage[pdftex,dvipsnames]{xcolor}  
\usepackage{fixltx2e}
\usepackage[colorinlistoftodos,prependcaption,textsize=tiny, disable]{todonotes}
\newcommandx{\unsure}[2][1=]{\todo[linecolor=red,backgroundcolor=red!25,bordercolor=red,#1]{#2}}
\newcommandx{\change}[2][1=]{\todo[linecolor=blue,backgroundcolor=blue!25,bordercolor=blue,#1]{#2}}
\newcommandx{\info}[2][1=]{\todo[linecolor=OliveGreen,backgroundcolor=OliveGreen!25,bordercolor=OliveGreen,#1]{#2}}
\newcommandx{\improvement}[2][1=]{\todo[linecolor=Plum,backgroundcolor=Plum!25,bordercolor=Plum,#1]{#2}}
\newcommandx{\thiswillnotshow}[2][1=]{\todo[disable,#1]{#2}}
\newtheorem{theorem}{Theorem}
\DeclareMathOperator{\Tr}{Tr}

\newcommand*{\mylength}{1.6in}
\newcommand*{\mylen}{1.6in}
\newcommand*{\figbarlen}{1.6in}
\title{\LARGE \bf
A Convex Polynomial Force-Motion Model for Planar Sliding:\\ Identification and Application 
}

\author{Jiaji Zhou, Robert Paolini,  J. Andrew Bagnell and Matthew T. Mason  \\
The Robotics Institute, Carnegie Mellon University \\
\{jiajiz, rpaolini, dbagnell, matt.mason\}@cs.cmu.edu
}

\begin{document}
\message{ !name(lc_learning.tex) !offset(-2) }

\maketitle
\thispagestyle{empty}
\pagestyle{empty}

\begin{abstract}
We propose a polynomial force-motion model for planar sliding. 
The set of generalized friction loads is the 1-sublevel set of a polynomial whose gradient directions correspond to generalized velocities. 
Additionally, the polynomial is confined to be convex even-degree homogeneous in order to obey the maximum work inequality, symmetry, shape invariance in scale, and fast invertibility.
We present a simple and statistically-efficient model identification
procedure using a sum-of-squares convex relaxation. 
Simulation and robotic experiments validate the accuracy and efficiency of our approach. 
We also show practical applications of our model including stable pushing of objects and free sliding dynamic simulations. 

\end{abstract}

\section{INTRODUCTION}
We develop a data-driven but physics-based method for modeling planar friction.    
Manipulations employing friction are ubiquitous in tasks including positioning and orienting objects by pushing \cite{Mason1986a, Lynch1996e, Akella1998, Dogar2010pgd}, controlled slip with dexterous hands \cite{Cole1992a} and assembly of tight-fitting parts \cite{Whitney1983b}.
In the case of planar robot pushing, indeterminacy of the
pressure distribution between the object and support surface leads to
uncertainty in the resultant velocity given a particular push
action. Despite such inherent difficulty, algorithms and analysis have been developed with provable guarantees. Mason \cite{Mason1986a} derived the voting theorem to determine the sense of rotation of an object pushed by a point contact. Lynch and Mason \cite{Lynch1996e} developed a stable pushing strategy when objects remain fixed to the end effector with two or more contact points.    
However, minimal assumptions on friction conditions inherently lead to 
conservative strategies. By explicitly modeling and identifying the friction space, \todo{Done. load space?}
we can improve strategies for planning and control. 
Our contribution lies in developing a precise and statistically-efficient
(i.e., requiring only a few collected force-velocity data pairs) model with a computationally efficient identification procedure. Fig. \ref{fig:pipeline} illustrates an outline of the paper.
We assume a quasi-static regime \cite{Mason1986} where forces and moments are balanced with negligible inertia effects. 
\begin{figure}[!h]
\centering
\includegraphics[width=\columnwidth]{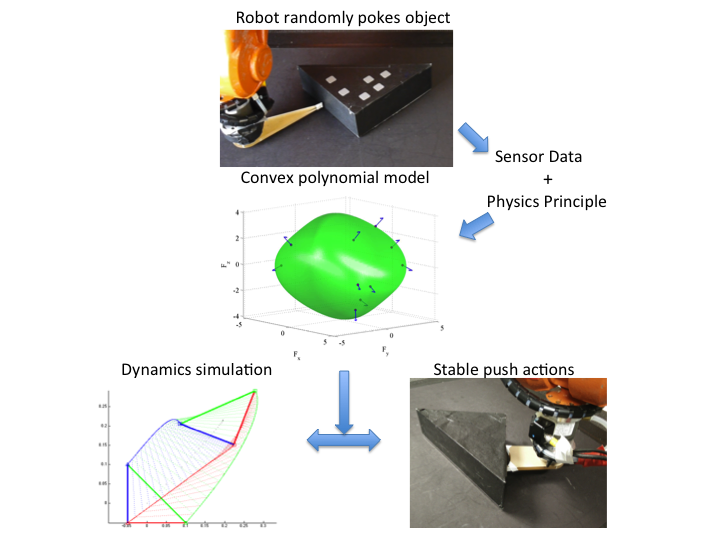}
\caption{\small The robot randomly pokes the object of known shape with a
  point finger to collect force-motion data. We then optimize a
  convex polynomial friction representation with physics-based constraints. Based on the representation, we
  demonstrate applications of stable pushing and
  dynamic sliding simulation.}
\label{fig:pipeline}
\vspace{-0.2in}
\end{figure}

\section{Background on Planar Force-Motion Models}
The classical Coulomb friction law states that for a point contact with
instantaneous planar velocity \todo{Done. Why direction?} $\mathbf{v} = [v_x, v_y]^T$, the incurred
friction force $\mathbf{f} =
[f_x, f_y]^T$ the point \todo{contact} applies on the surface is parallel to $\mathbf{v}$, i.e., $\mathbf{f}/{|\mathbf{f}|} = \mathbf{v}/{|\mathbf{v}|}$. 
We refer the readers to \cite{Mason1986a} for details of friction analysis for
planar sliding under isotropic Coulomb friction law.
In this paper, we build our analysis on a generalized friction law formulated
first in \cite{moreau1988unilateral}, in which $\mathbf{v}$
and $\mathbf{f}$ may not be parallel, but only need to obey the
maximum work inequality: \todo{This expression is for convenience of the follow-up discussion of rigid body extension.JAB: make the maximum part explicit. Sup or max.}
\begin{align}
(\mathbf{f} - \mathbf{f'}) \cdot \mathbf{v} \geq 0,  \label{eq:point_max_power}
\end{align}
where $\mathbf{f'}$ is an arbitrary element from the set of all possible
static and kinetic friction forces. 

Let $\mathbf{V} = [V_x, V_y, \omega]^T$ be the instantaneous generalized
velocity and $\mathbf{F} = [F_x, F_y, \tau]^T$ be the generalized friction load for a rigid body sliding on a planar surface with a contact area $R$.
Both $\mathbf{V}$ and $\mathbf{F}$ are in the local body frame\footnote{Throughout the paper, we use a local coordinate frame with the origin set as the projection of the COM onto the supporting surface. However, the choice of the origin can be any other point of convenience.}. $\mathbf{F}$ can be computed by integration over $R$:
\begin{align}
F_x  = \int_{R} f_{ax}\,da, \quad F_y & = \int_{R} f_{ay}\,da, \quad \tau = \int_{R} (r_{ax} f_{ay} - r_{ay} f_{ax})\,da.
\end{align}  
The maximum work inequality in equation
(\ref{eq:point_max_power}) can be extended for
generalized friction load $\mathbf{F}$ and velocity $\mathbf{V}$:
\begin{align}
\mathbf{F} \cdot \mathbf{V} & = \int_{R} f_{ax} (V_x - \omega r_{ay})\,da +   \int_{R} f_{ay} (V_y + \omega r_{ax}) \,da \nonumber \\
& = \int_{R} f_{ax} v_{ax} + f_{ay} v_{ay} \,da =  \int_{R} \mathbf{f}_{a}\cdot \mathbf{v}_{a} \,da \geq \mathbf{F'} \cdot \mathbf{V}, 
\end{align}
among any other possible friction load $\mathbf{F'}$.  
Due to the converse supporting hyperplane theorem \cite{boyd2004convex},
the set of all generalized friction loads form a convex set $\mathcal{F}$. 
An important work that inspires us is Goyal et al. \cite{Goyal1991} who found that 
all possible generalized friction loads during sliding form a
limit surface (LS) constructed from the Minkowsky sum of limit curves at
individual support points. Points inside the surface correspond to static friction loads. 
Points on the surface correspond to friction loads with normals parallel to sliding velocity directions,
forming a mapping between generalized friction load and sliding
velocity. An ideal LS is always convex due to the maximum
work inequality but may not be strictly convex when a single point
supports finite pressure. As shown in Fig. \ref{fig:robot_lc_ideal},
facets can occur since the object can rotate about one of the three
support points whose velocity is zero with indeterminate underlying friction.

Erdmann \cite{erdmann1994representation} proposed a configuration
space embedding of friction. In his work, the third component of $\mathbf{F}$ is $F_z = \tau/\rho$ and the third component of $\mathbf{V}$ is $V_z = \omega\rho$, where $\rho$ is the radius of gyration. In doing so, all three components in $\mathbf{F}$ and $\mathbf{V}$ have the same unit. Observe that such normalized representation also obeys maximum work inequality with $\rho$ being any characteristic length. In our experiments, we have found that the normalized representation yields better numerical condition and different values of $\rho$ including radius of gyration, average edge length and minimum enclosing circle radius lead to similar performance.

\section{RELATED WORK} 
Yoshikawa and Kurisu \cite{Yoshikawa1991} solved an unconstrained least-squares problem to estimate the center of friction and the pressure
distribution over discrete grids on the contact surface. 
With similar set up, Lynch \cite{Lynch1993} proposed a constrained linear 
programming procedure to avoid negative pressure assignment.
However, methods based on discretization of the support surface introduce two sources of error in both localization of support points and pressure assignment among those points. 
We do not need to estimate the exact location of
support points. Coarse discretization loses accuracy while fine discretization unnecessarily increases the dimensionality
of estimation and model complexity. Howe and Cutkosky
\cite{howe1996practical} presented an ellipsoid approximation of the
limit surface assuming known pressure distribution. The ellipsoid was
constructed by computing or measuring the major axis lengths (maximum
force during pure translation and maximum torque during pure
rotation). Facets can be added by intersecting the ellipsoid with
planes determined by each support point. The pressure distribution
(except for 3 points support with known center of pressure),
nevertheless, is non-trivial to measure or compute. We also show that the ellipsoid approximation, as a special
case of our convex polynomial representation, is less accurate due to lack of expressiveness. 

Recent data-driven attempts \cite{kopicki2011learning,
  omrcen2009autonomous} collected visual data from random push trials and applied ``off-the-shelf''
machine learning algorithms to build motion models.  
We also embrace a data-driven strategy but bear in mind that physics
principles should guide the design of the learning algorithm (as
constraints and/or priors), hence reducing sample complexity and
increasing generalization performance. 

\section{REPRESENTATION AND IDENTIFICATION}
In this section, we propose the sublevel set representation of
friction with desired properties and show that convex even-degree homogeneous polynomials are valid solutions. Then we formulate an efficient convex optimization procedure to identify such polynomials.
\subsection{Polynomial sublevel set representation}
Let $H(\mathbf{F})$ be a differentiable convex function that models the generalized friction load and velocity as follows: 
\begin{itemize}
\item The $1$-sublevel set $L_{1}^{-}(H) = \{\mathbf{F}: H(\mathbf{F}) \leq 1\}$ corresponds to the convex set $\mathcal{F}$ of all generalized friction loads. 
\item The $1$-level set $L_{1}(H) = \{\mathbf{F}: H(\mathbf{F})=1\}$
  corresponds to generalized friction loads (during slip) on the boundary surface of $\mathcal{F}$.
\item The surface normals given by gradients $\{\nabla H(\mathbf{F}): \mathbf{F} \in L_{1}(H)\}$ represent instantaneous generalized velocity directions during slip, i.e.,
$V = s\nabla H(\mathbf{F})$ where $s>0$.  
\end{itemize}

\begin{theorem} \label{theorem:levelset}
The set of friction loads represented by the $1$-sublevel set of a differentiable convex function follows the maximum work inequality.
\end{theorem}
\begin{proof}
 When the object remains static, $\mathbf{F}$ belongs to the interior of $L_{1}^{-}(H)$ and $\mathbf{V}$ equals zero, the inequality holds
 as equality. When the object slips, $\mathbf{F}\in
 L_{1}(H)$ and $\mathbf{V}$ is nonzero,  we have for any other generalized friction load $\mathbf{F'}\in L_{1}^{-}(H)$:
\begin{align}
  \mathbf{V}\cdot(\mathbf{F'} - \mathbf{F}) = s(\nabla
  H(\mathbf{F})\cdot(\mathbf{F}' - \mathbf{F}))  \nonumber 
  \leq s(H(\mathbf{F}') - H(\mathbf{F})) \leq 0, 
\end{align}
where the first inequality is due to the convexity of $H(\mathbf{F})$.      
\end{proof}

In addition to enforcing convexity (discussed in \ref{sec:sos_convex}), we choose $H(\mathbf{F})$ to obey the following properties:
\begin{enumerate}
\item Symmetry: $H(\mathbf{F}) = H(-\mathbf{F})$ and $\nabla H(\mathbf{F}) = - \nabla H(-\mathbf{F})$.
\item Scale invariance: $\nabla H(a\mathbf{F}) = g(a)\nabla H(\mathbf{F})$, where $g(a)$ is a positive scalar function. 
\item Efficient invertibility: there exists efficient numerical procedure to find a $\mathbf{F}\in L_{1}(H)$ such that $\nabla H(\mathbf{F})/ \|\nabla H(\mathbf{F})\|  = \mathbf{V}$ for a given query unit velocity $\mathbf{V}$. We denote such operation as $\mathbf{F} = H_{inv}(\mathbf{V})$.
\end{enumerate}
Symmetry is based on the assumption that negating the velocity
direction would only result in a sign change in the friction
load. Scale invariance is desired for two reasons: 1) scaling in mass
and surface coefficient of friction could only result in a change of
scale but not other geometrical properties of the level-set representation;
and 2) predicting directions of generalized velocities (by computing
gradients and normalizing to a unit vector) only depends on the
direction of generalized force. Such a property is useful in the
context of pushing with robot fingers where applied loads are
represented by friction cones. The inverse problem of finding the
friction load for a given velocity naturally appears in seeking
quasi-static balance for stable pushing or computing deceleration during free sliding,
as shown in Section \ref{sec:application}. In general, efficient
numerical solution to the inverse problem, which our representation
enables, is key to planning and simulation.  One solution family for $H(\mathbf{F})$ that obeys these properties is the set of \textbf{strongly convex even-degree homogeneous polynomials}.

\begin{theorem} \label{theorem:homogcvx}
A strongly convex even degree-$d$ homogeneous polynomial $H(\mathbf{F};a)
=\sum_{i=1}^{m} a_iF_x^{i_1}F_y^{i_2}F_z^{d-i_1-i_2}$ with $m$ monomial terms\footnote{The number of different monomial
  terms $m$ is bounded by $d+2 \choose 2$.} parametrized by $a$
satisfies the properties of symmetry, scale invariance, and efficient invertibility.
\end{theorem}

\begin{proof}
Proving symmetry and scale invariance are trivial due to the homogeneous and
even-degree form of $H(\mathbf{F})$. 
Here, we sketch the proof that efficient invertibility can be achieved
by first solving a simple non-linear least square problem followed by a rescaling.

Construct an objective function $G(\mathbf{F}) =
\frac{1}{2}\|\nabla H(\mathbf{F}) - \mathbf{V}\|^2$ whose gradient
$\frac{\partial G}{\partial \mathbf{F}} = \nabla^2H(\mathbf{F})
(\nabla H(\mathbf{F}) - \mathbf{V})$. Note that its
stationary point $F^*$, which iterative methods such as Gauss-Newton or trust-region algorithms will converge to, satisfies $\nabla H(\mathbf{F^*}) - \mathbf{V} = 0$. 
Hence $F^{*}$ is globally optimal with value zero.
Let $\Delta \mathbf{F}_t = {\nabla^2 H(\mathbf{F}_t)}^{-1}(\mathbf{V}_t - \mathbf{V})$, then the update rule for Gauss-Newton algorithm is $\mathbf{F}_{t+1} = \mathbf{F}_{t} - \Delta \mathbf{F}_t$. 
Although the final iteration point $\mathbf{F}_T$ may not lie on the $1$-level set of $H(\mathbf{F})$, we can scale $\mathbf{F}_T$ by $\hat{\mathbf{F}}_T = H(\mathbf{F}_T)^{-1/d} \mathbf{F}_T$ such that $H(\hat{\mathbf{F}}_T) = 1$ and $\nabla H(\hat{\mathbf{F}}_t) / \|\nabla H(\hat{\mathbf{F}}_t)\| = V$ due to the homogeneous form of $H(\mathbf{F})$. Therefore $H_{inv}(\mathbf{V}) = \hat{\mathbf{F}}_T$.
\end{proof}

\subsection{Sum-of-squares Convex Relaxation} \label{sec:sos_convex}
Enforcing strong convexity for a degree-$2$ homogeneous polynomial $H(\mathbf{F};A) =
\mathbf{F}^TA\mathbf{F}$ has a straightforward set up
as solving a semi-definite programming problem with constraint of $A\succeq \epsilon I$. 
Meanwhile, for a polynomial of degree greater
than 2 whose hessian matrix $\nabla^2H(\mathbf{F};a)$ is a function of both $\mathbf{F}$ and $a$, certification of positive semi-definiteness is NP-hard. However, recent progress \cite{parrilo2000structured,magnani2005tractable} in sum-of-squares programming has given powerful semi-definite relaxations of global positiveness certification of polynomials.    
Specifically, let $\mathbf{z}$ be an arbitrary non-zero vector in $\mathbb{R}^3$ and $y(\mathbf{F},\mathbf{z}) = [z_1F_x,z_1F_y, z_1F_z, z_2F_x, z_2F_y, z_2F_z, z_3F_x, z_3F_y, z_3F_z]^T$. If
there exists a positive-definite matrix $Q$ such that   
\begin{align}
\mathbf{z}^T\nabla^2 H(\mathbf{F};a) \mathbf{z}  &= y(\mathbf{F},\mathbf{z})^T Q y(\mathbf{F},\mathbf{z}) > 0, \label{eq:sos-matrix}
\end{align}
then $\nabla^2H(\mathbf{F};a)$ is positive definite
for all non-zero $\mathbf{F}$ under parameter $a$ and $H(\mathbf{F};a)$ is called as sos-convex. Further, equation (\ref{eq:sos-matrix}) can be
written as a set of $K$ sparse linear constraints on $Q$ and $a$.
\begin{align} 
\Tr(A_{k} Q) &= b_k^Ta , \quad k\in \{1 \dots K\} \nonumber \\
Q &\succeq \epsilon I,
\end{align}
where $A_k$ and $b_k$ are constant sparse element indicator matrix and
vector that only depend on the polynomial degree $d$. The number of
constraints $K$ equals 27 for $d = 4$. 

\subsection{Identification} \label{sec:calibration}
This section sets up an efficient convex optimization
for identifying the coefficient $a$ of the polynomial $H(\mathbf{F};a)$ given a set of
measured noisy generalized force-motion $\{\mathbf{F}_{i\in \{1 \dots N\}},
\mathbf{V}_{i\in \{1 \dots N\}}\}$ pairs. In our experiments, we use homogeneous 4th order polynomial. 
The optimization should find the coefficient $a$ such that the
measured forces $\mathbf{F}_i$ are close to the 1-level set surface and the
corresponding gradients are aligned well (up to scale) w.r.t measured velocities
$\mathbf{V}_i$. 
Let $\alpha_i = ||\nabla H(\mathbf{F}_i;a) - (\nabla H(\mathbf{F}_i;a)\cdot V_i) V_i||_2^2$ be 
the L2-projection residual of $\nabla H(\mathbf{F}_i;a)$ onto the measured unit velocity vector $V_i$, and let 
$\beta_i =  (H(\mathbf{F}_i;a) - 1)^2$ be a distance measurement of $\mathbf{F}_i$ from the 1-level set of $H(\mathbf{F}_i;a)$.  
We set up the optimization as follows: 
\begin{align}
  & \underset{a,Q}{\text{minimize}}
  & & \|a\|_2^2 + \sum_{i=1}^{N}(\eta_1 \alpha_i + \eta_2 \beta_i) \label{eq:opt_goal}\\
  & \text{subject to}
  & & \Tr(A_kQ) = b_k^Ta,  \; k = 1,\ldots,K, \label{eq:opt_lrcon}\\
  &&& Q \succeq \epsilon I. \label{eq:opt_sdpcon}
\end{align}
The first term is for parameter regularization.
$\eta_1$ and $\eta_2$ are trade-off parameters determined by cross-validation.
Equations (\ref{eq:opt_lrcon}) and (\ref{eq:opt_sdpcon}) enforce
convexity. Note that the objective is quadratic in $a$ with sparse
linear constraints and a semi-definite constraint on $Q$.\footnote{Code link: \scriptsize{\url{https://github.com/robinzhoucmu/MLab_EXP/blob/master/SlidingExpCode/LimitSurfaceFit/Fit4thOrderPolyCVX.m}}}
We would like to point out that the formulation can be adapted online using projected gradient descent so that the importance of historical data is diminishing as the object moves, enabling the estimation to adapt to changing surface conditions. Evaluating such online version of the identification algorithm is deferred to future work.  

\section{Experiments}
 We conduct simulation and robotic experiments to
demonstrate the accuracy and statistical-efficiency of our proposed representation. 
The model converges to a good solution with few
available data which saves experimental time and design efforts.
We compare the following four different force-motion model representations $\mathcal{H}$:
1) degree-4 convex homogeneous polynomial (poly4-cvx); 
2) degree-4 homogeneous polynomial (poly4) with convexity constraints 
3) convex quadratic (quad) as degree-2 polynomial, i.e.,  $H(\mathbf{F}) =
  \mathbf{F}^TA \mathbf{F}$ with ellipsoid sublevel set;
and 4) gaussian process (GP) with squared exponential 
  kernel\footnote{The squared exponential kernel gives better 
    performance over linear and
    polynomial. Normalizing the input load to a unit vector improves
    performance by requiring the GP to ignore scale. Every
    ($\mathbf{F},\mathbf{V}$) input pair is augmented with ($-\mathbf{F}, -\mathbf{V}$) for training.}.
Denote by $\mathbf{V}_i$ the ground truth instantaneous generalized velocity direction and
$\mathbf{V}_p(\mathbf{F}_i;\mathcal{H})$ as the predicted generalized velocity
direction based on $\mathcal{H}$ for the input generalized
load $\mathbf{F}_i$, we use the average angle $\delta(\mathcal{H}) = \frac{1}{N}
\sum_{i=1}^{N} \arccos (\mathbf{V}_p(F_i;\mathcal{H}) \cdot \mathbf{V}_i) $ between $\mathbf{V}_p(F_i;\mathcal{H})$ and
$\mathbf{V}_i$ as an evaluation criterion.

\subsection{Simulation Study}
Two kinds of pressure distribution are studied. 
\begin{itemize}
\item ``Legged'' support: Randomly sampled three support points on a unit
circle with randomly assigned pressure. 
\item ``Uniform'' support: Uniformly distributed 360 support points on
  a unit circle and 400 support points within a unit square. Each point has the
  same support pressure.
\end{itemize}
For each pressure configuration, we conduct 50 experimental
trials. 
To generate the simulated force-motion data, we assume a Coulomb friction model at each support point with a uniform coefficient of friction.
Without loss of generality, sum of pressure over all contact points is normalized to one and the origin is set as the center of pressure. 
 For each trial of ``uniform'' support, we sampled 150 instantaneous generalized
velocities directions $\mathbf{V}_i$ uniformly
on the unit sphere and compute the corresponding generalized friction
loads $\mathbf{F}_i$. 
For each trial of ``legged'' support, 75 $(\mathbf{F}_i,\mathbf{V}_i)$ pairs are uniformly sampled on the facets (same $\mathbf{V}_i$ but different $\mathbf{F}_i$ for each facet) and another 75 pairs are uniformly sampled in the same fashion as ``uniform'' support. In doing so, the dataset has a diverse coverage.
Among the 150 pairs, 50\% is used for
hold-out testing, 20\% is used for cross validation and four
different amounts (7, 15, 22, 45) from the rest of 30\% are used as training.
In order to evaluate the algorithms' robustness under noise, we additionally
corrupt the training and validation set using Gaussian noise of standard deviation $\sigma = 0.1$ to each dimension
of both $\mathbf{F}_i$ and $\mathbf{V}_i$ (renormalized to unit vector).
From Fig. \ref{fig:simbar} we can reach the following conclusions. 
1) Poly4-cvx has the smallest $\delta(\mathcal{H})$ for different amounts of training data and pressure configurations.
2) Both poly4-cvx and convex quadratic show superior performance when data is scarce and noisy, demonstrating convexity is key to data-efficiency and robustness. 
Poly4-cvx model additionally shows larger improvement as more data is available due to stronger model expressiveness. 
3) Poly4 (without convexity constraint) performs the worst when only few data is available, but gradually improves as more data
is available for shaping the surface.\footnote{For noise-free experiments shown in Fig. \ref{fig:simbar_2} and \ref{fig:simbar_4}, 
when enough training data (more than 22) is presented, poly4 performs slightly better than poly4-convex. 
We conjecture such difference is due to the gap between sos-convex polynomials and convex polynomials.} GP has similar performance trends as poly4 but worse on average.
4) Polynomial models enjoy significant performance advantages when limit surface is smoother as in uniform point support (approximation of uniform patch contact). 
Such advantage is smaller for three-points support whose limit surface has large flat facets.     

\begin{figure*}[t!]
\centering
\begin{subfigure}[t]{\figbarlen}
\centering
\includegraphics[width=\figbarlen]{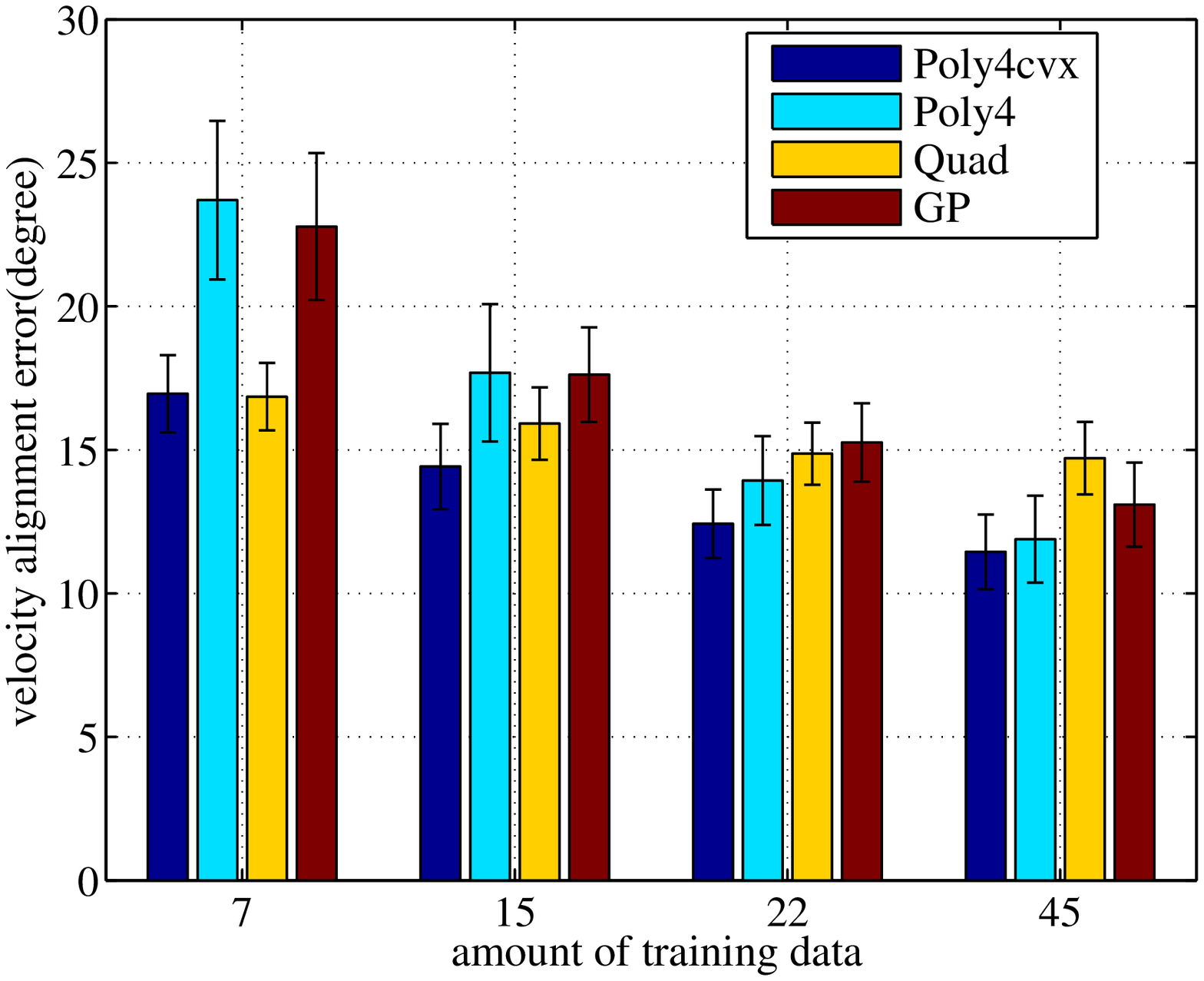} 
\caption{Three support points with noisy training and validation data.\\}
\label{fig:simbar_1}
\vspace{-0.1in}
\end{subfigure}
~
\begin{subfigure}[t]{\figbarlen}
\centering
\includegraphics[width=\figbarlen]{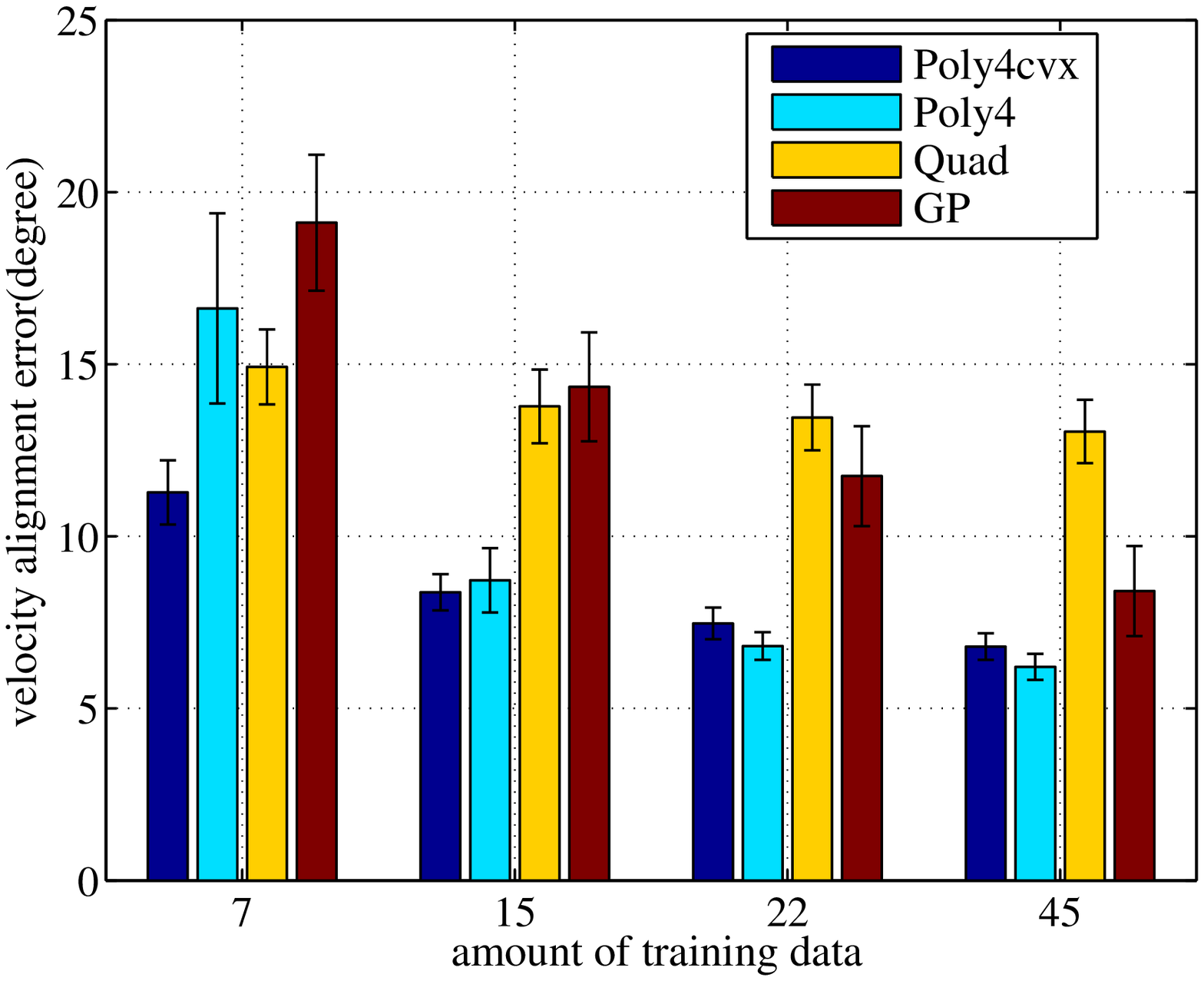}
\caption{Three support points with noise-free training and validation data.\\}
\label{fig:simbar_2}
\vspace{-0.1in}
\end{subfigure}
~
\begin{subfigure}[t]{\figbarlen}
\centering
\includegraphics[width=\figbarlen]{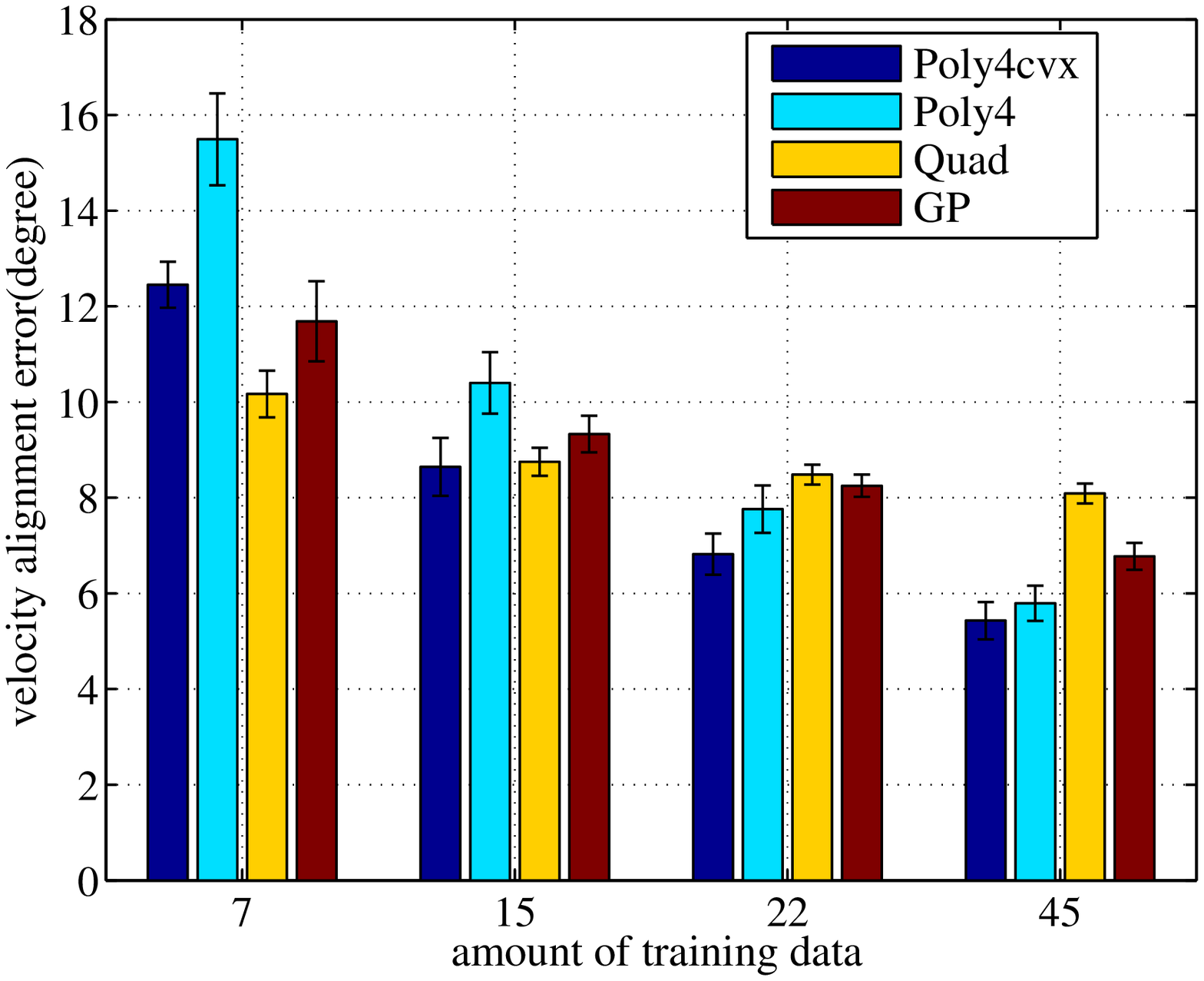} 
\caption{Uniform circular support points with noisy training and validation data.\\}
\label{fig:simbar_3}
\vspace{-0.1in}

\end{subfigure}
~
\begin{subfigure}[t]{\figbarlen}
\centering
\includegraphics[width=\figbarlen]{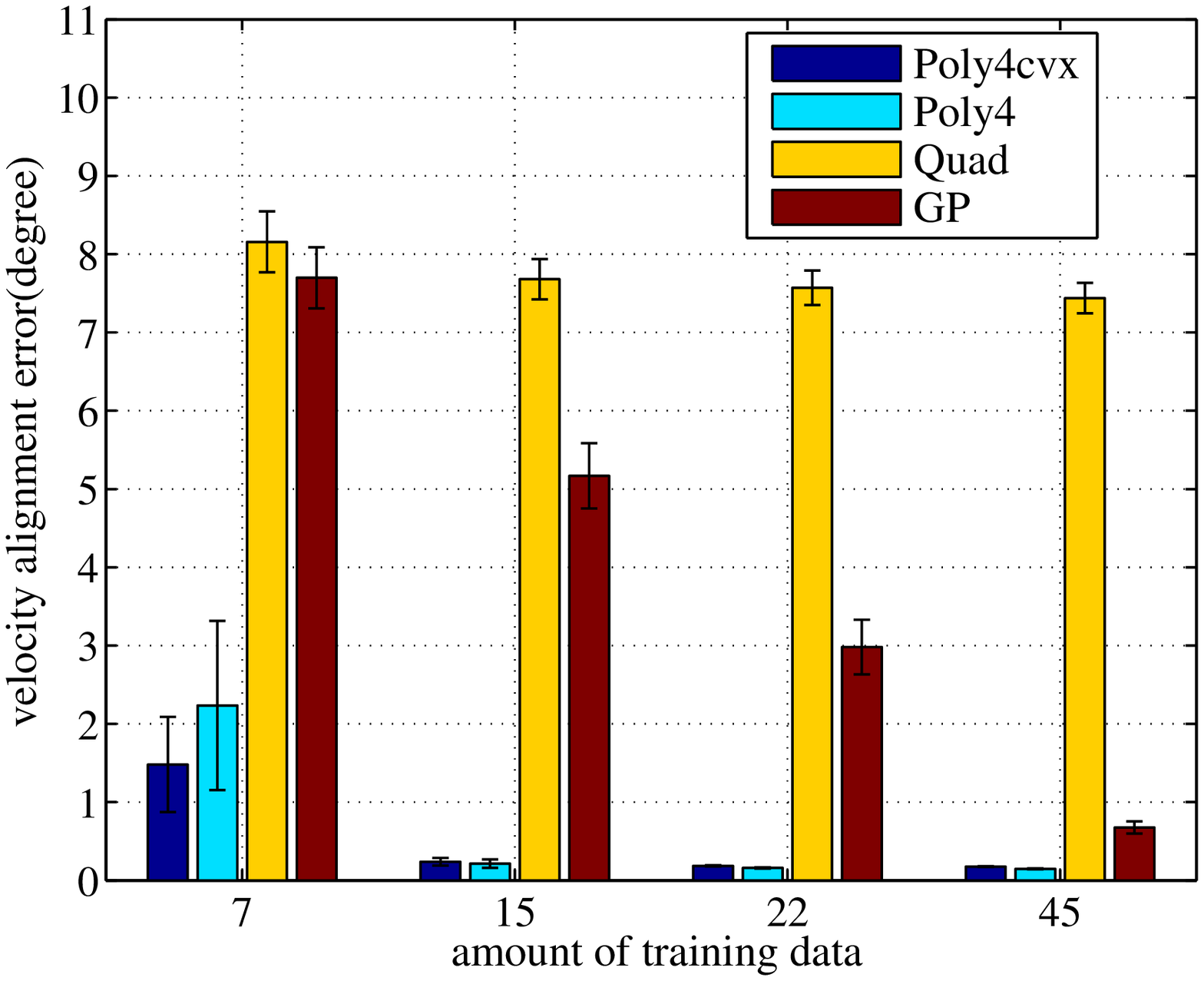}
\caption{Uniform circular support points with noise-free training and validation data.}
\label{fig:simbar_4}
\vspace{-0.1in}
\end{subfigure}
\normalsize
\caption{\small Test error comparison for simulation experiments with 95\% confidence bar
  (50 random evaluations) among different methods as amount of
  training data increases for three random support points and 360 support points on a
  ring respectively. Results for uniform pressure distribution within a square are similar to
  uniform circular support and omitted for space.}
\label{fig:simbar}
\vspace{-0.1in}
\end{figure*}

\subsection{Robotic Experiment}
We mount three screws at four different sets of locations underneath an alluminium right-angle triangular work
object\footnote{The triangular object weighs 1.508kg with edge lengths of 150mm, 150mm and 212.1mm. 
The four different set of support point locations (in mm) with respect
to the right angle corner vertex are: [(10,10), (10,130), (130,10)],
[(30,30), (30,90), (90,30)], [(10,10), (10,130), (90,30)], [(30,30), (63.33,43.33), (43.33,63.33)].}. 
Given known mass and COM projection, ideal ground truth pressure for
each support point can be computed by solving three linear
equations assuming each screw head approximates a point
contact. Fig. \ref{fig:robot_lc_tri} shows a flipped view of one
arrangement whose ideal LS is illustrated in
Fig. \ref{fig:robot_lc_ideal}, constructed by Minkowski addition of generalized friction at each
single point support assuming Coulomb friction model with uniform coefficient of friction. 
Three pairs of symmetric facets\footnote{The
third one is in the back not visible from presented view.} 
characterize indeterminate friction force when
rotating about one of the three support points.
Comparison among identified fourth-order homogeneous polynomials with and without convexity constraint is shown in Fig. \ref{fig:robot_lc_poly4_5} and \ref{fig:robot_lc_poly4cvx_5}. 
We can see that convex-shape constraint is essential to avoid poor generalization error when little data is available.
Fig. \ref{fig:robot_lc_qp_10} and \ref{fig:robot_lc_poly4cvx_10}
compare the level sets of a convex quadratic (ellipsoid) and
a sos-convex degree-4 homogeneous polynomial, demonstrating that
the higher degree polynomial captures the facets effect better than quadratic models.
\begin{figure}[!h]
\centering
\begin{subfigure}[t]{1.5in}
\centering
\includegraphics[width=1.5in]{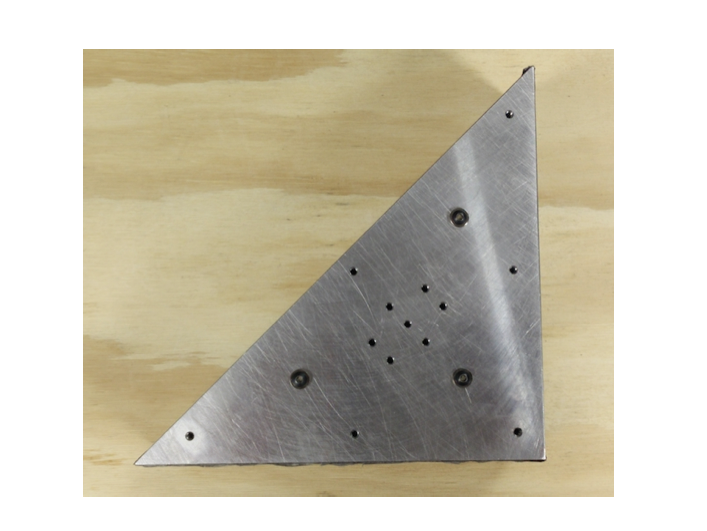} 
\caption{Triangular block with three support screws.}
\label{fig:robot_lc_tri}
\end{subfigure}
~
\begin{subfigure}[t]{1.7in}
\centering
\includegraphics[width=1.7in]{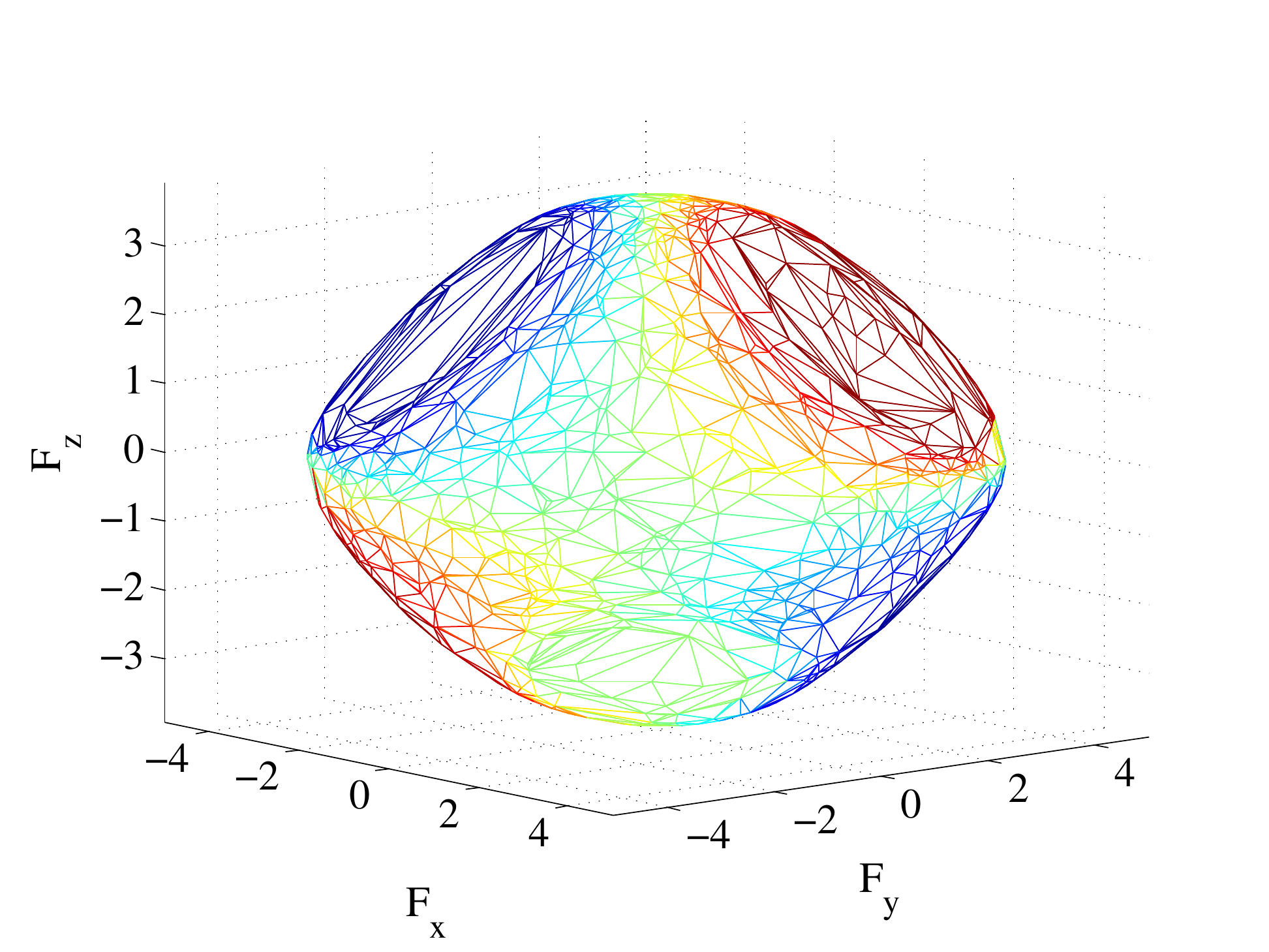}
\caption{Ideal limit surface with facets.}
\label{fig:robot_lc_ideal}
\end{subfigure}
~
\begin{subfigure}[t]{\mylength}
\centering
\includegraphics[width=\mylength]{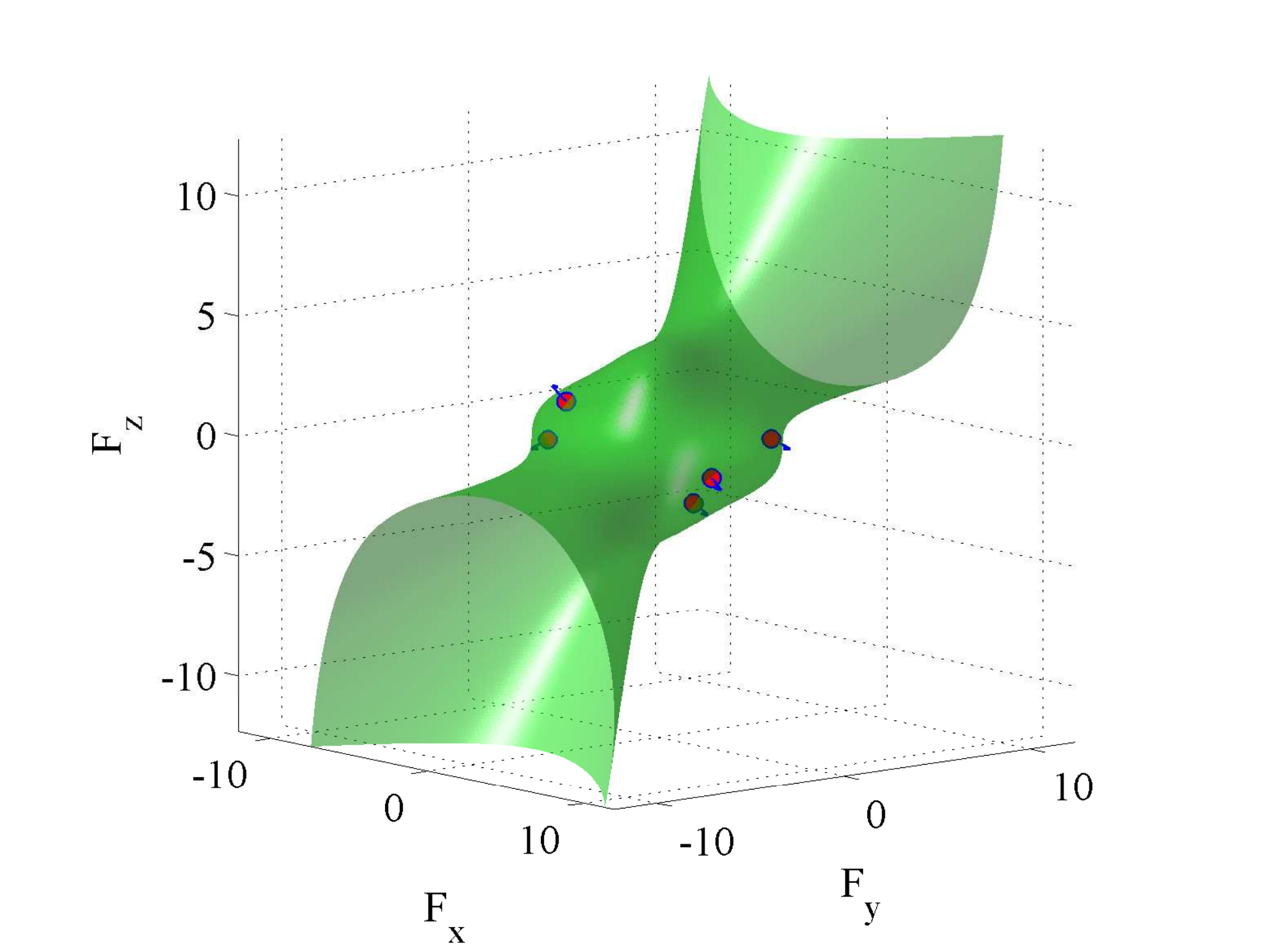}
\caption{Poly4 fit with 5 training and 5
  validation data. }
\label{fig:robot_lc_poly4_5}
\end{subfigure}
~
\begin{subfigure}[t]{\mylength}
\centering
\includegraphics[width=\mylength]{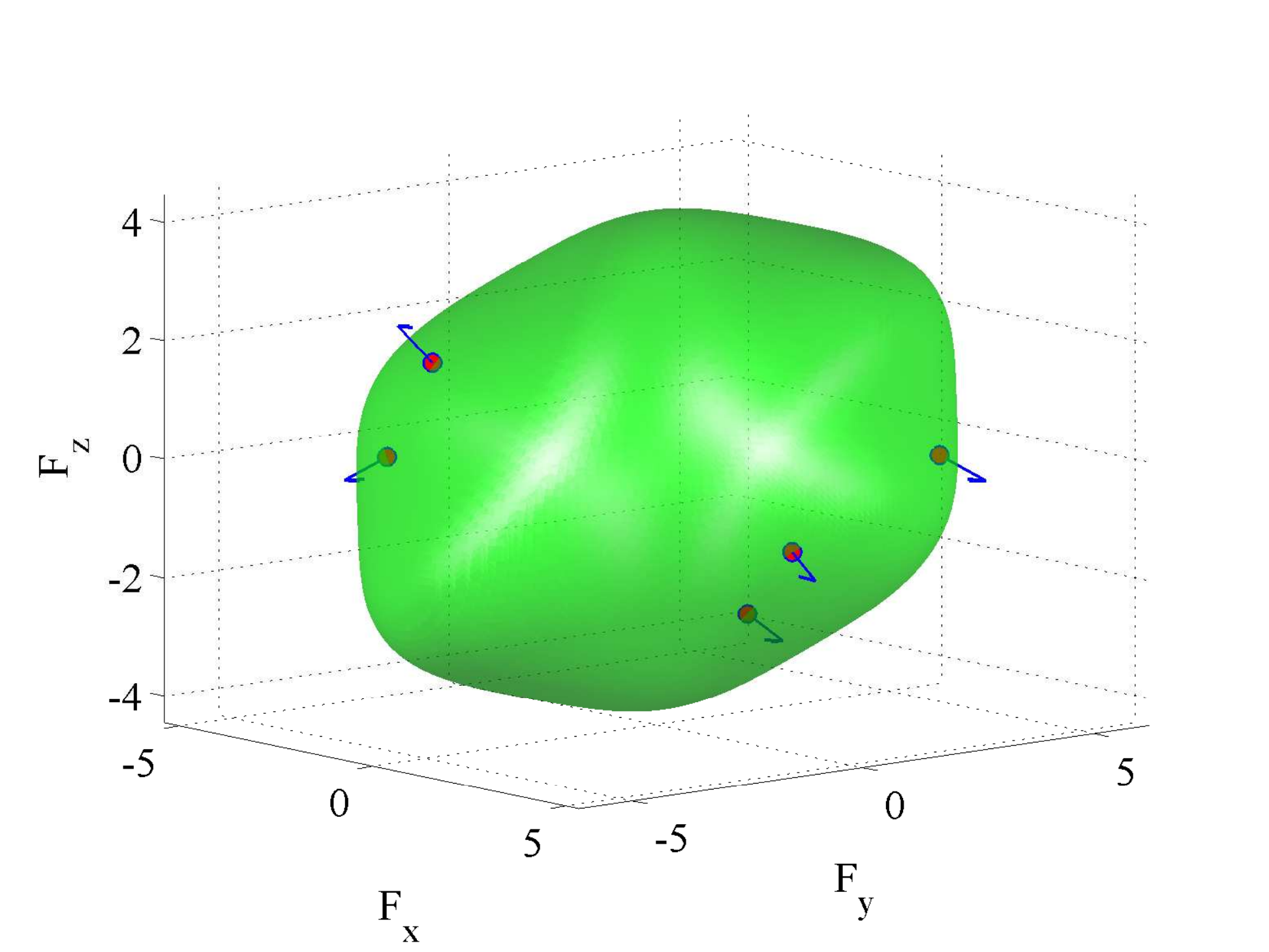} 
\caption{Poly4-cvx fit with 5 training and 5
  validation data.}
\label{fig:robot_lc_poly4cvx_5}
\end{subfigure}
~
\begin{subfigure}[t]{\mylength}
\centering
\includegraphics[width=\mylength]{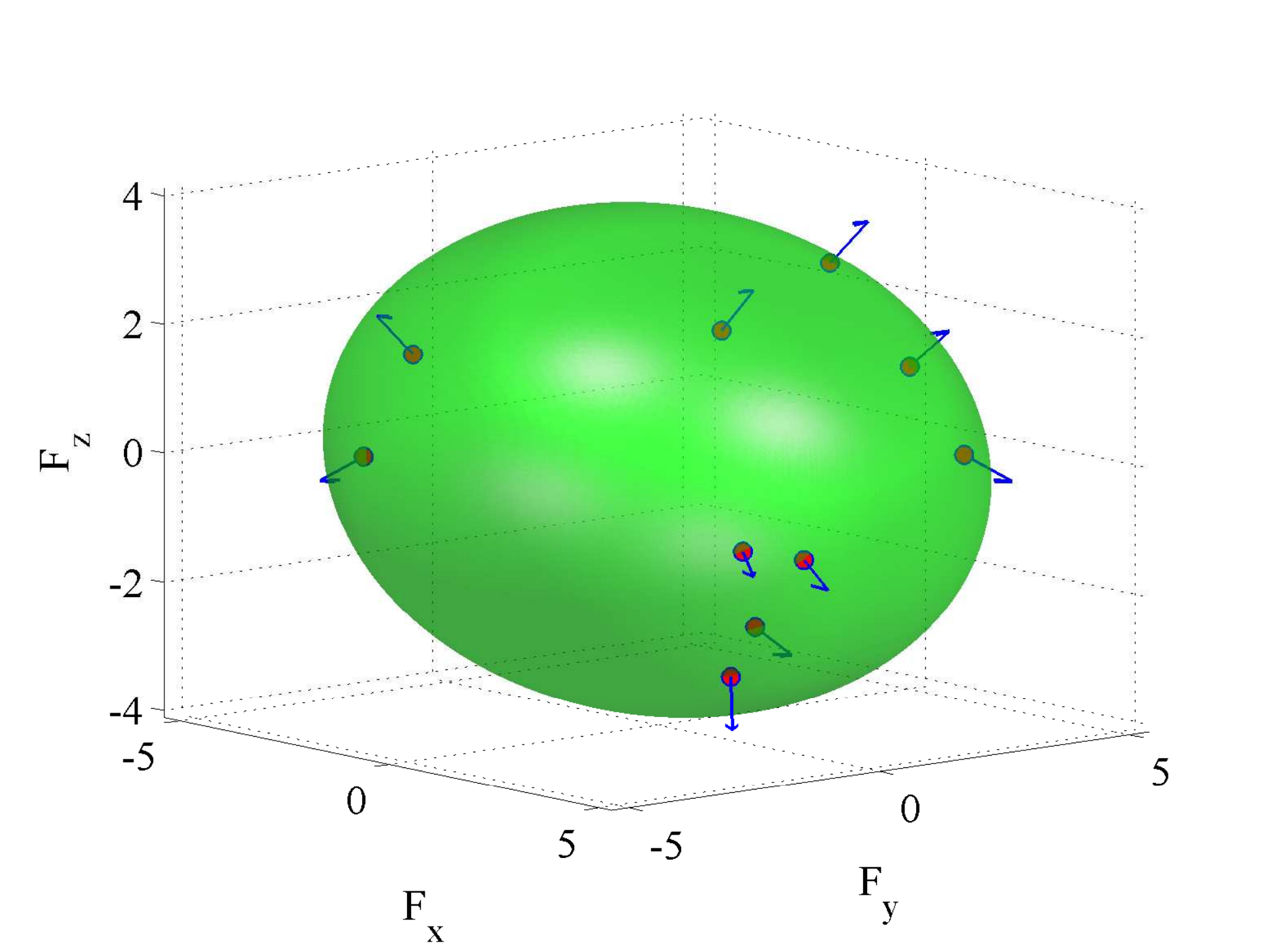} 
\caption{Convex quadratic fit with 10 training and 10
  validation data.\\}
\label{fig:robot_lc_qp_10}
\vspace{-0.1in}
\end{subfigure}
~
\begin{subfigure}[t]{\mylength}
\centering
\includegraphics[width=\mylength]{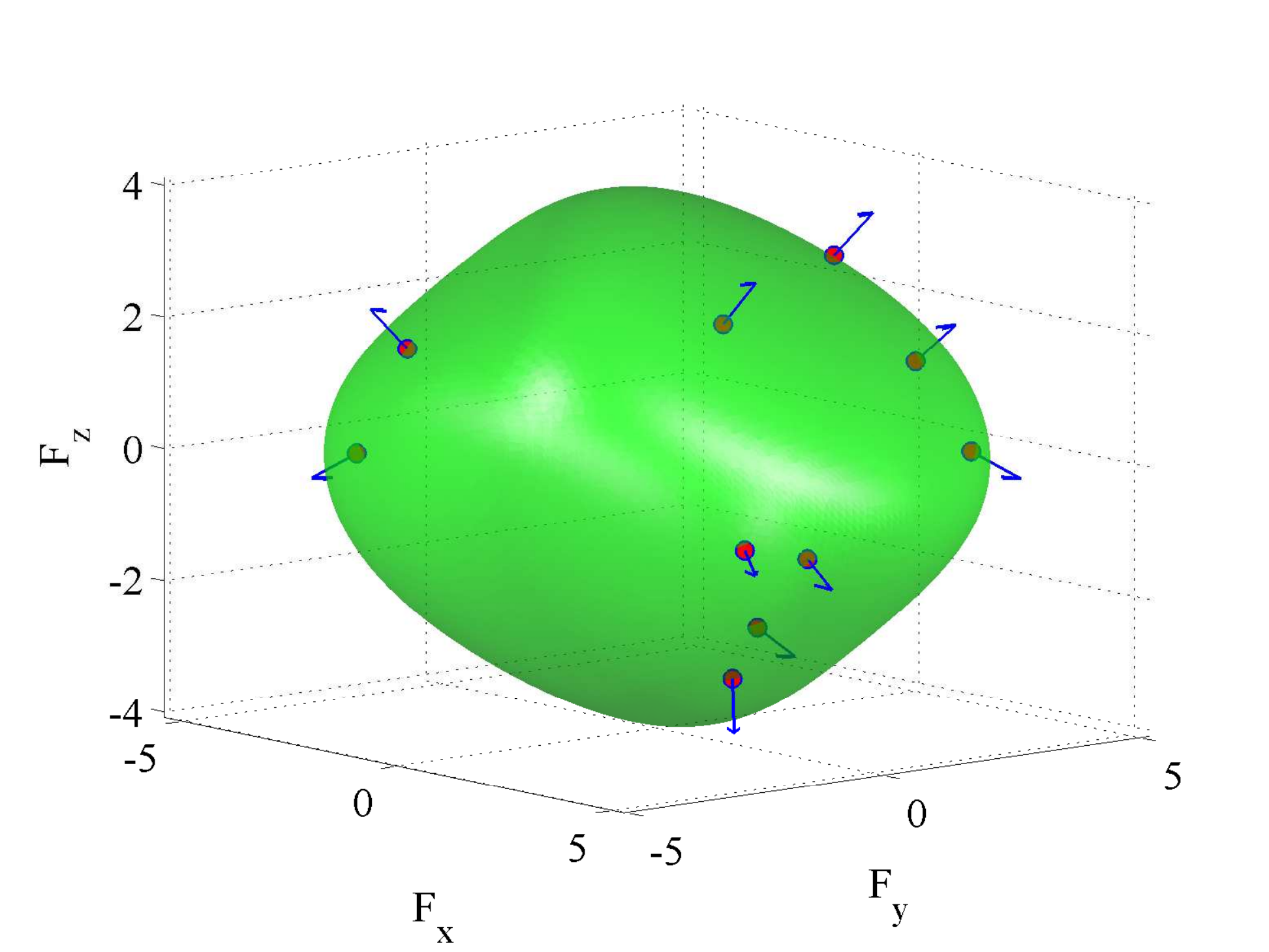} 
\caption{Poly4-cvx fit with 10 training and 10
  validation data.}
\vspace{-0.1in}
\label{fig:robot_lc_poly4cvx_10}
\end{subfigure}

\caption{\small Level set friction representations for the pressure arrangement in Fig. \ref{fig:robot_lc_tri}.   
Red dots and blue arrows are collected generalized forces and
velocities from force-torque and motion capture sensor
respectively. Fig. \ref{fig:robot_lc_poly4_5} and 
\ref{fig:robot_lc_poly4cvx_5}, Fig. \ref{fig:robot_lc_qp_10} and
\ref{fig:robot_lc_poly4cvx_10} share the same data.}
\label{fig:robot_lc}
\vspace{-0.15in}
\end{figure}

We conduct robotic poking (single point pushing) experiments on wood and paper board surfaces. 
In each experiment, we generate 50 pokes (30 for training set, 10 for validation set and 10 for test set) with randomly chosen contact points and
pushing velocity directions.\footnote{During each pushing action, the robot moves at a
slow speed of 2.5mm/s with a total small push-in distance of 15mm. 
Each generalized velocity direction is approximated as the direction of pose
displacement and generalized force is averaged over the action duration.} 
Fig. \ref{fig:robot_exp_compare} shows model accuracy (averaged over four different pressure arrangments) with
respect to increase in amount of training data for
different methods evaluated on both the hold-out test sensor data and
samples from ideal LS. We can see similar performance trends as
in simulation experiments. Note that both evaluations only serve as 
certain reference criteria. Sensor data is noisy and all possible force measurements
from a single point pusher only cover a
limited space of the set of friction loads. We also do not
intend to treat the idealized limit surface as absolute ground truth
as there is no guarantee on uniform coefficient of friction between
the support points and the underlying surface. Additionally, point
contact and isotropic Coulomb friction model are only approximations
of reality. Nevertheless, both evaluations demonstrate performance advantage
of our proposed poly4-cvx model.
\small
\begin{figure*}[t!]
\centering
\begin{subfigure}[t]{\mylen}
\centering
\includegraphics[width=1.6in]{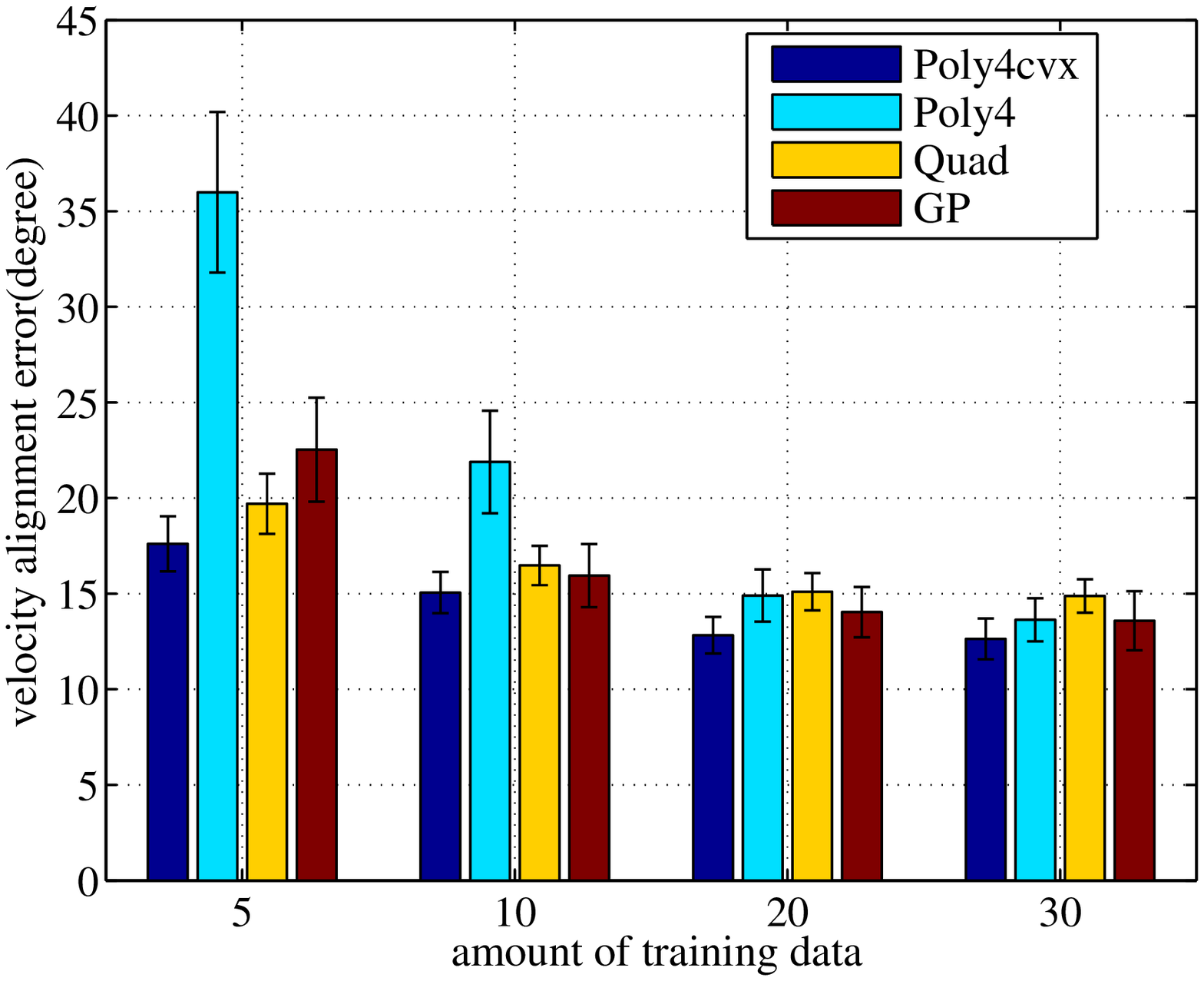} 
\caption{\small Test on sensor data (wood surface). \\}
\vspace{-0.1in}
\label{fig:robot_exp_wood}
\end{subfigure}
~
\begin{subfigure}[t]{\mylen}
\centering
\includegraphics[width=1.6in]{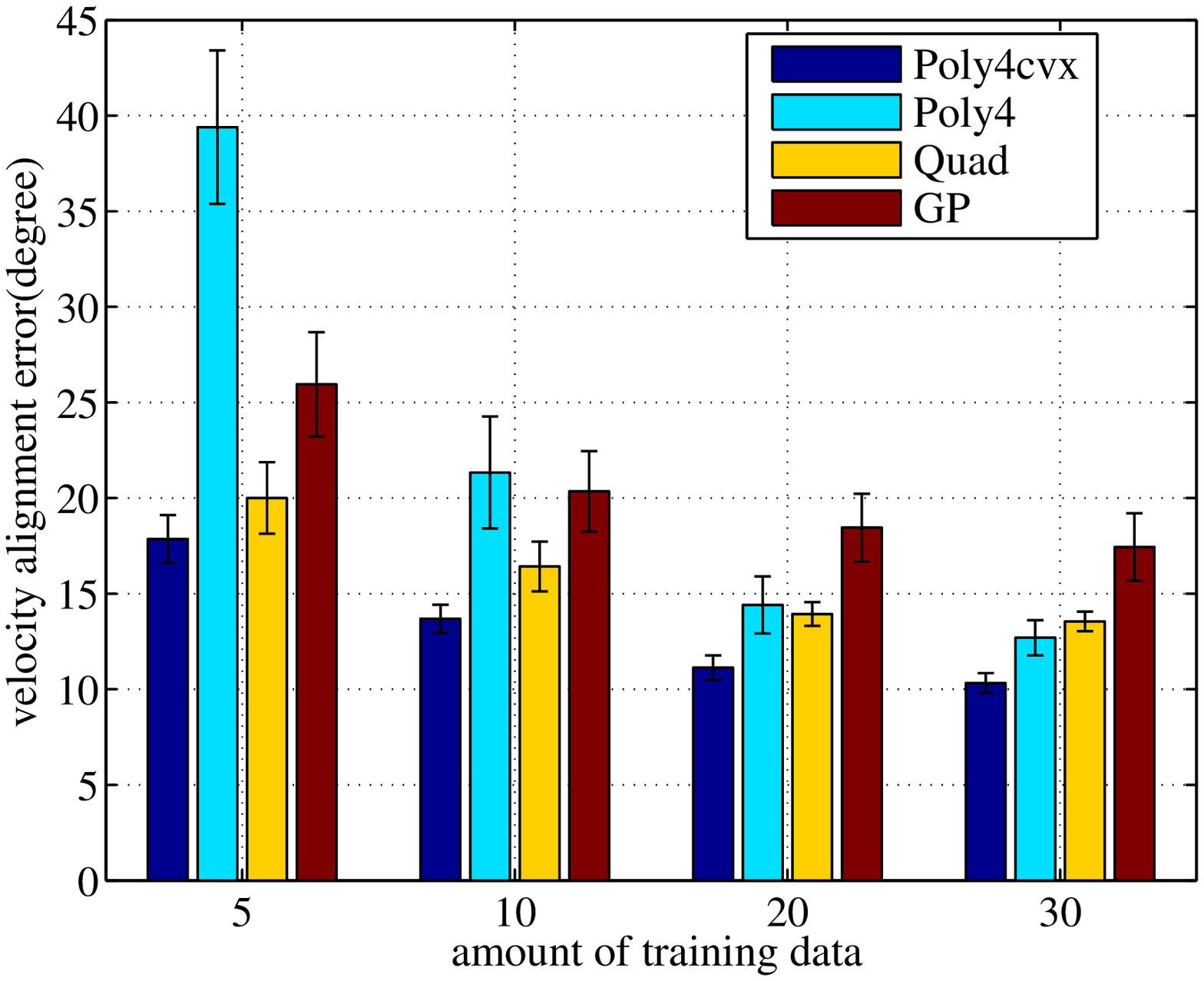}
\caption{\small Test on data sampled from ideal LS (wood surface).\\}
\vspace{-0.1in}
\label{fig:robot_sim_wood}
\end{subfigure} 
~
\begin{subfigure}[t]{\mylen}
\centering
\includegraphics[width=1.6in]{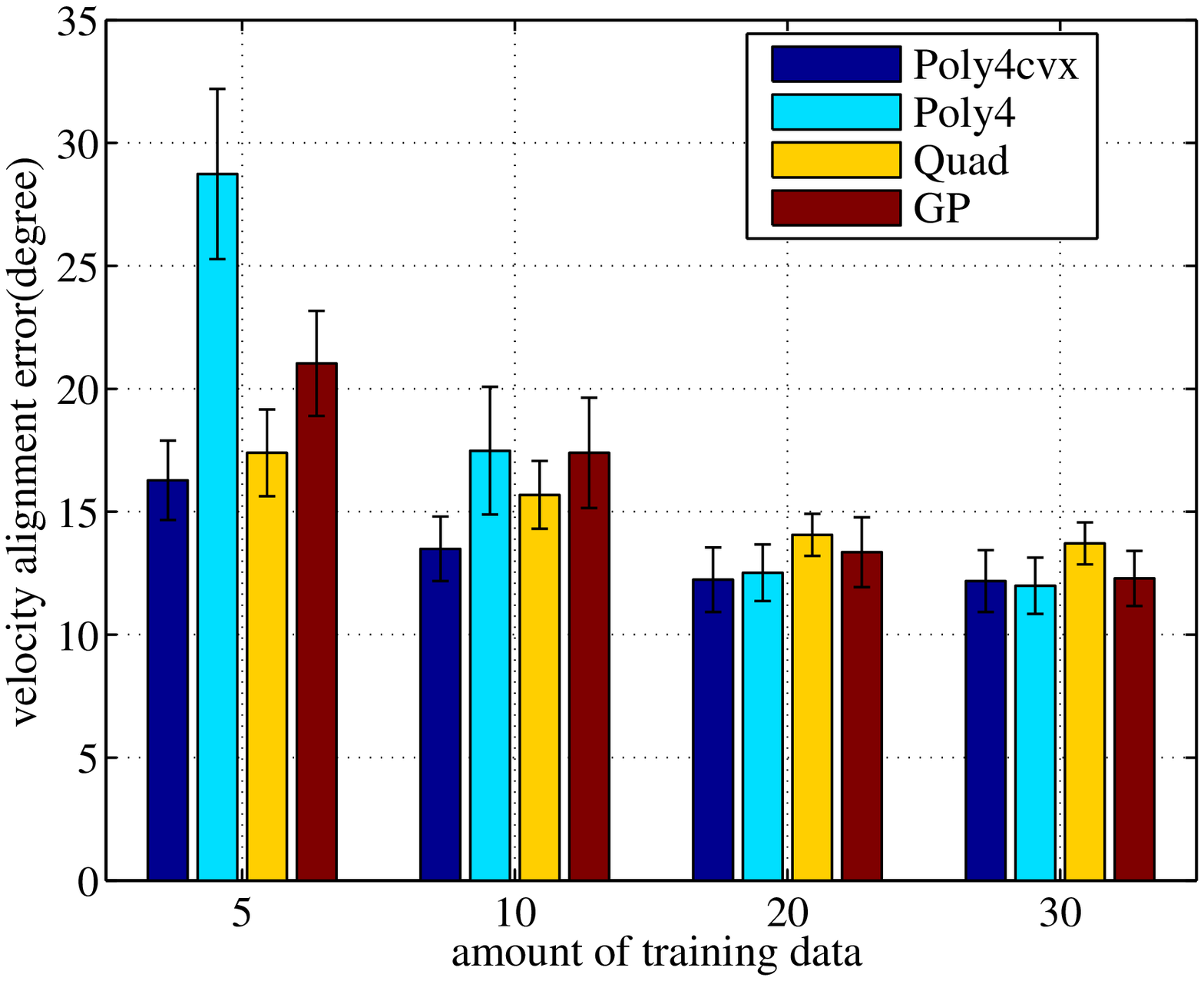}
\caption{\small Test on sensor data (paper board surface). \\}
\vspace{-0.1in}
\label{fig:robot_exp_blk}
\end{subfigure} 
~
\begin{subfigure}[t]{\mylen}
\centering
\includegraphics[width=1.6in]{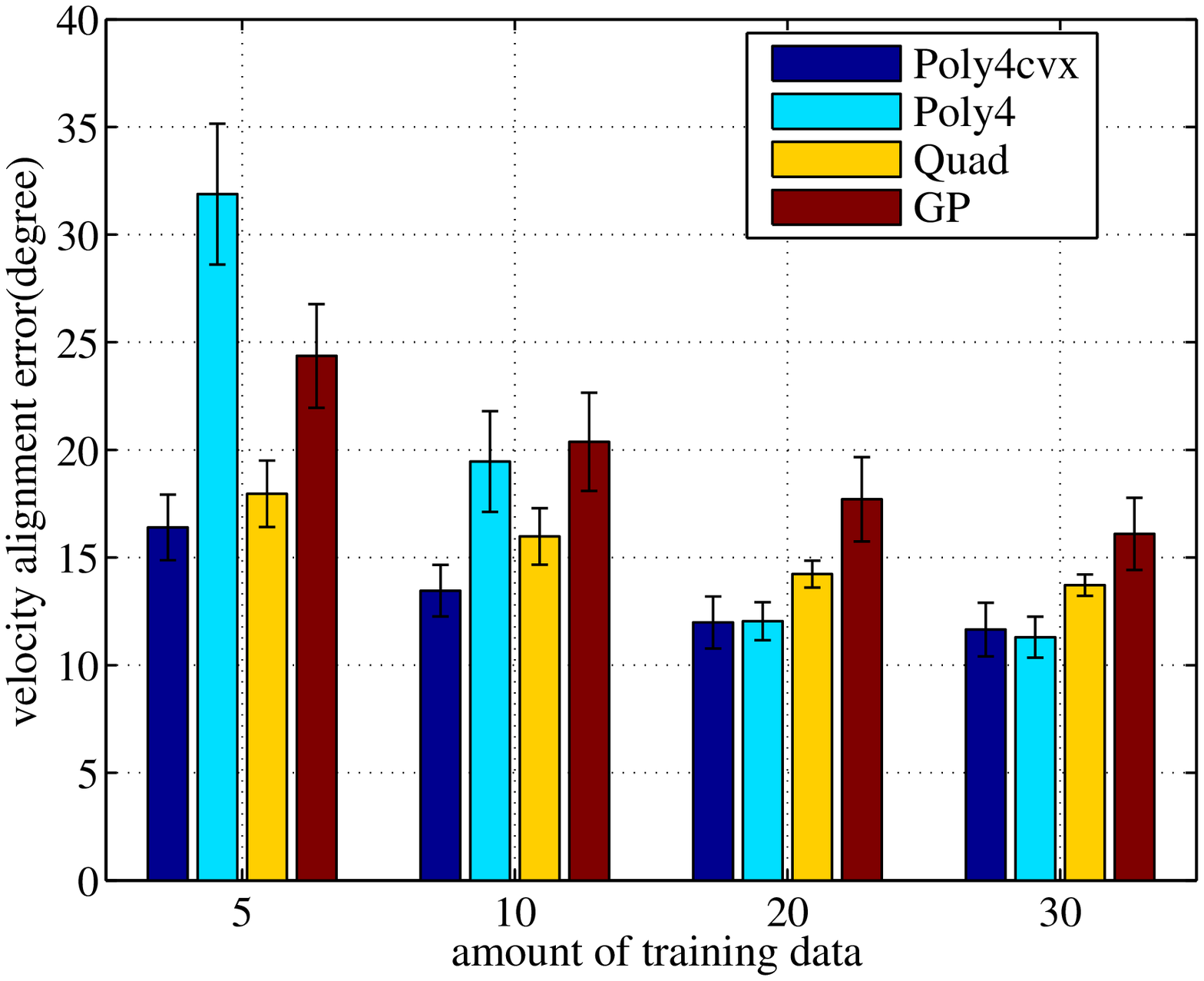}
\caption{\small Test on data sampled from ideal LS (paper board surface).}
\vspace{-0.1in}
\label{fig:robot_sim_blk}
\end{subfigure} 

\caption{\small Test error comparison for robotic experiments with 95\% confidence bar (50 random evaluations) among
  different methods as amount of training data increases for three support
  points on wood and hard paper board surfaces.}
\label{fig:robot_exp_compare}
\vspace{-0.15in}
\end{figure*}
\normalsize
\section{Applications}\label{sec:application} 
\subsection{Stable Push Action Generation}
The resultant object velocity under a single point
push action is hardly fully predictable. However a two-points push action against an edge of the object
can be stable such that the object will remain attached to the pusher without slipping or breaking contact \cite{Lynch1996e}. That is, the slider and pusher will move about the same center of rotation (COR) point $p_{c}$. 
Given the level set representation $H(\mathbf{F})$, the condition of determining whether a two-points push with instantaneous generalized velocity $\mathbf{V}_{p_c}$ is stable or not is equivalent to check if the corresponding generalized friction force
$\mathbf{F}_{p_c} = H_{inv}(\mathbf{V}_{p_c})$ lies in the
applied composite generalized friction cone $\mathbf{F}_c$. To validate predictions based on the model, we sampled 60 random CORs and execute with the robot for three different pressure arrangements on a novel support surface material (hard poster paper).\footnote{We use the same triangular block in Fig. \ref{fig:robot_lc_tri} with two three-points contacts [(10,10), (10,130), (130,10)], [(30,30), (30,90), (90,30)] as well as full patch contact. The 60 CORs are tight rotation centers within a 400mm$\times$400mm square centered at the COM.} 15 out of the 60 CORs are labelled as stable. The training force-motion data are collected from pushing the object on a wood surface. 
Table \ref{table:stable_acc} and \ref{table:stable_recall} summarize the classification accuracy and positive (stable) class recall measurements of three invertible methods 
with respect to increase in amount of training data. Fig.
\ref{fig:stable_region} shows an example (full patch contact) that the stable regions generated from
the identified poly4-cvx model is much larger than the conservative
analysis as in \cite{Lynch1996e} which misses the tight/closer rotation centers.
\begin{table}
\centering
\caption{\small Comparison of average accuracy with 95\% confidence
  interval as amount of training data increases.}
\begin{tabular}{|c|c|c|c|}
\hline
 & 10 & 20 & 30 \\
\hline
poly4-cvx & \textbf{88.13}$\pm$1.80 & \textbf{91.33}$\pm$1.61 & \textbf{93.07}$\pm$1.45 \\
poly4 & 85.27$\pm$2.12 & 89.40$\pm$1.98 & 93.00$\pm$1.62  \\
quadratic & 87.93$\pm$1.72 & 87.20$\pm$1.65 & 88.00$\pm$1.39  \\
\hline
\end{tabular}
\label{table:stable_acc}
\end{table}
\vspace{-5pt}
\begin{table}
\centering
\caption{\small Comparison of average positive recall with 95\%
  confidence interval as amount of training data increases.}
\begin{tabular}{|c|c|c|c|}
\hline
 & 10 & 20 & 30 \\
\hline
poly4-cvx & \textbf{90.13}$\pm$3.54 & \textbf{96.69}$\pm$1.93 & \textbf{98.18}$\pm$1.32 \\
poly4 &79.96$\pm$5.25 & 92.76$\pm$2.90 & 97.18$\pm$1.84 \\
quadratic &73.18$\pm$4.61 &73.38$\pm$4.69 & 73.87$\pm$4.63 \\
\hline
\end{tabular}
\label{table:stable_recall}
\vspace{-0.15in}
\end{table}

\normalsize
\begin{figure}[h!]
\centering
\includegraphics[width=3in]{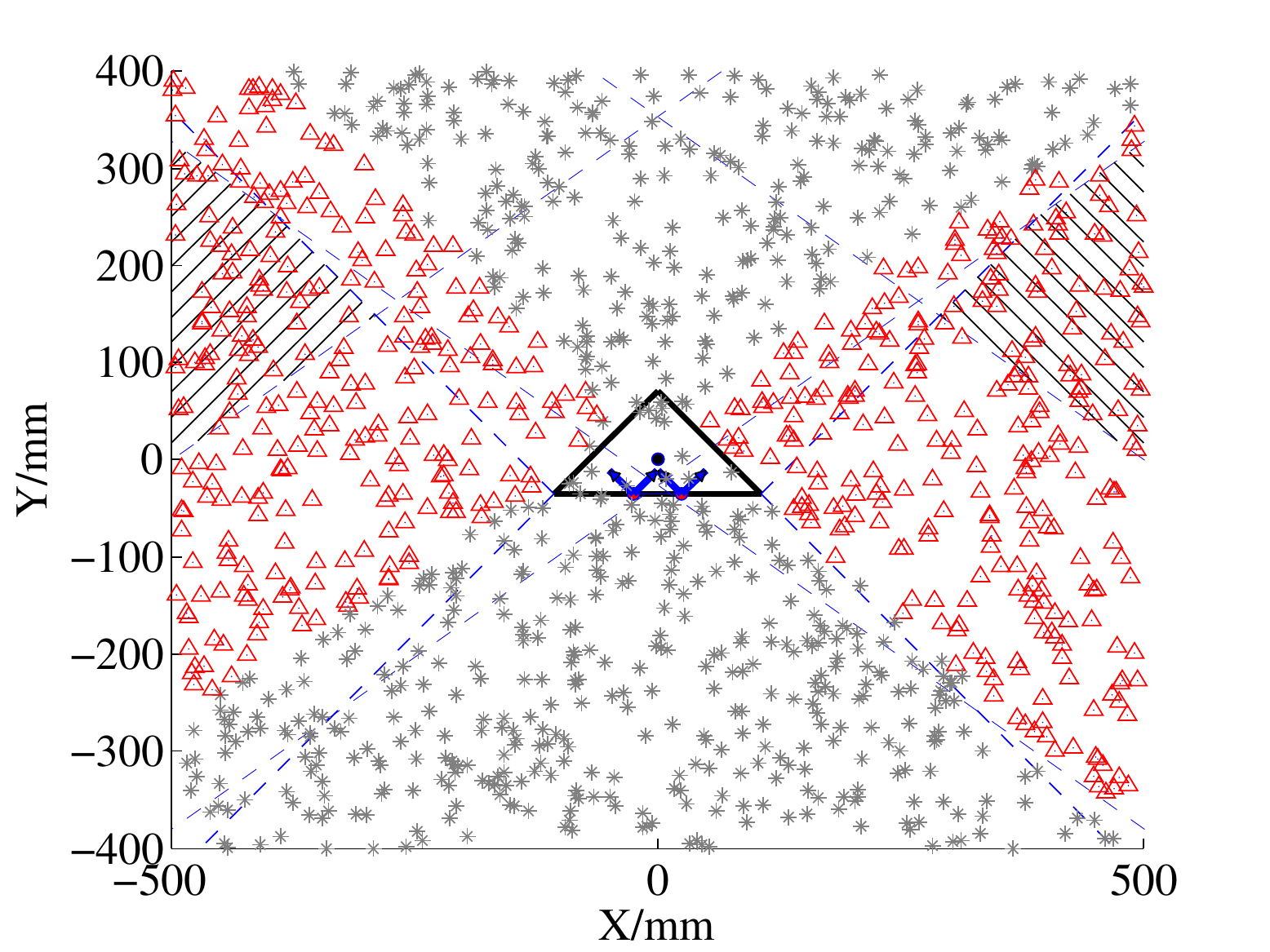}
\vspace{-5pt}
\caption{\small Hatched areas correspond to stable CORs region based
  on the conservative analysis \cite{Lynch1996e}. Red triangles are stable CORs and
gray stars are non-stable CORs based on the poly4-cvx model. The two push points are 50mm
in width. The pusher and objects are covered with electrical tape and guffer tape respectively with measured coefficient of friction equals one.}
\label{fig:stable_region}
\vspace{-0.15in}
\end{figure}

\subsection{Free Sliding Dynamics Simulation}
Given $H(\mathbf{F})$, the equation of motion (with respect to the
object local coordinate frame) during free sliding assuming a uniform surface can be written as $I\frac{dV}{dt} = -H_{inv}(\mathbf{V})$,
where $I$ is the moment of inertia. We use the Runge-Kutta method provided
in MATLAB ODE45 and demonstrate several example sliding trajectories in
Fig. \ref{fig:dyn}. As studied in \cite{Goyal1991a}, given an ideal
limit surface, a free sliding object comes at rest with one of several
definite generalized velocity directions (in local body frame), termed as
eigen-directions. We have empirically found a similar trend that there
exists multiple converging sets of
initial generalized velocities. An example
behavior is shown in Fig. \ref{fig:dynlc} where the final velocity directions
(instantaneous rotation centers) remain in the same
small region regardless of different initial velocities. 

\begin{figure}[h!]
\centering
\begin{subfigure}{\mylen}
\centering
\includegraphics[width=1.6in]{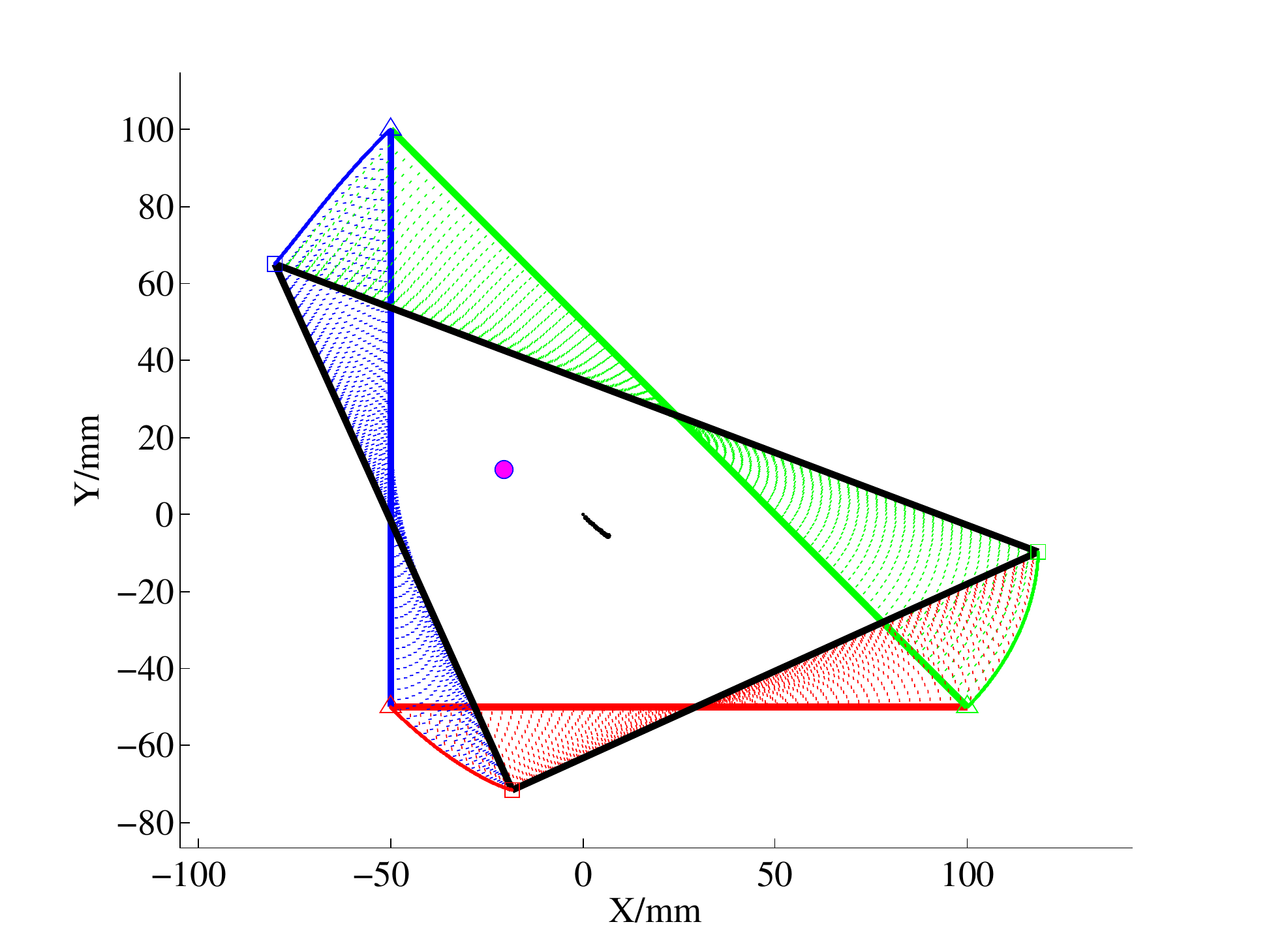} 
\caption{\small Trajectory with $\mathbf{V}(0)$ equals (150mm/s, -150mm/s, 2$\pi$rad/s). }
\label{fig:dyn1}
\end{subfigure}
~
\begin{subfigure}{\mylen}
\centering
\includegraphics[width=1.6in]{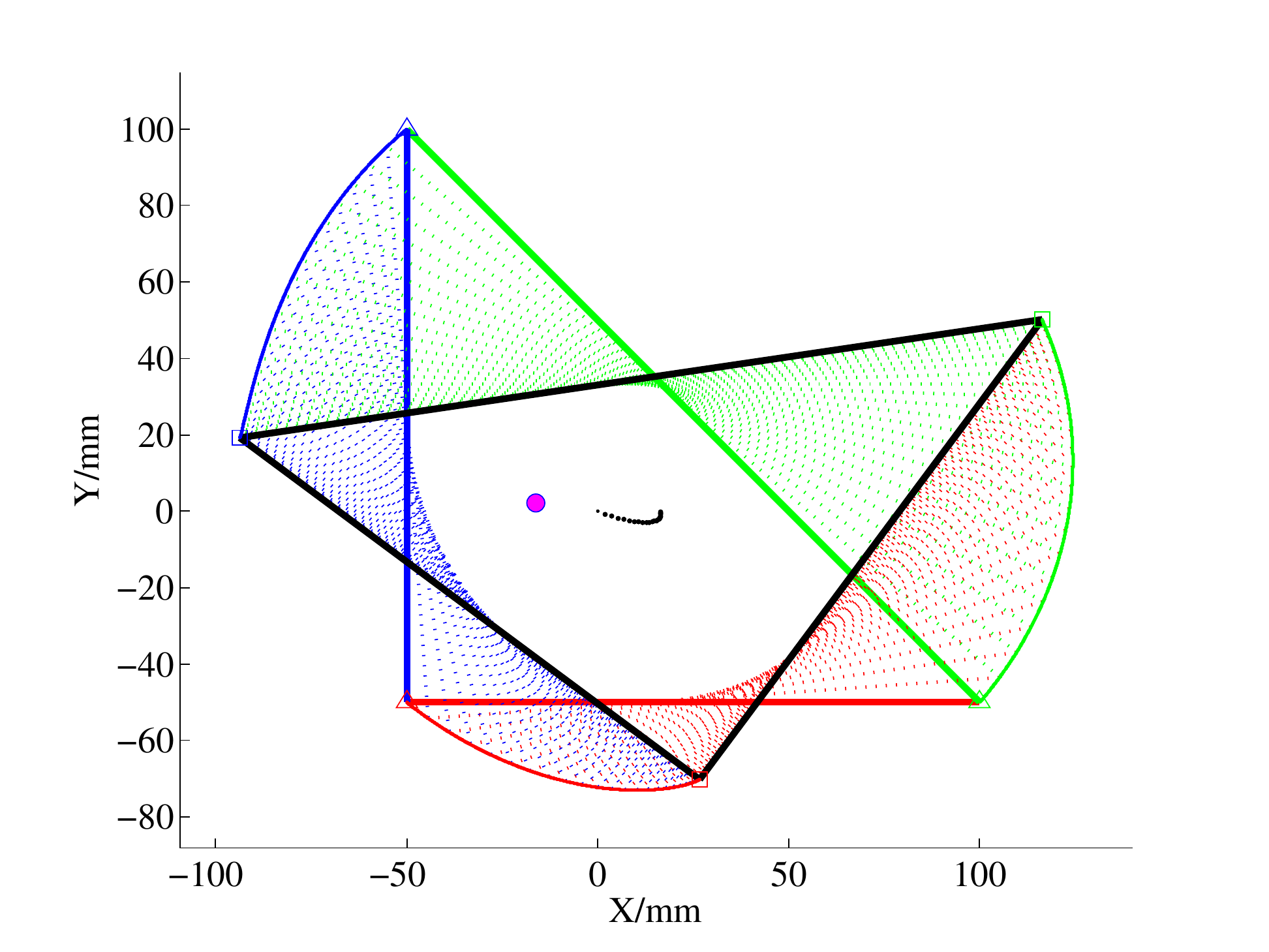}
\caption{\small Trajectory with $\mathbf{V}(0)$ equals (250mm/s, -100mm/s, 3$\pi$rad/s). }
\label{fig:dyn2}
\end{subfigure} 
~
\begin{subfigure}{\mylen}
\centering
\includegraphics[width=1.6in]{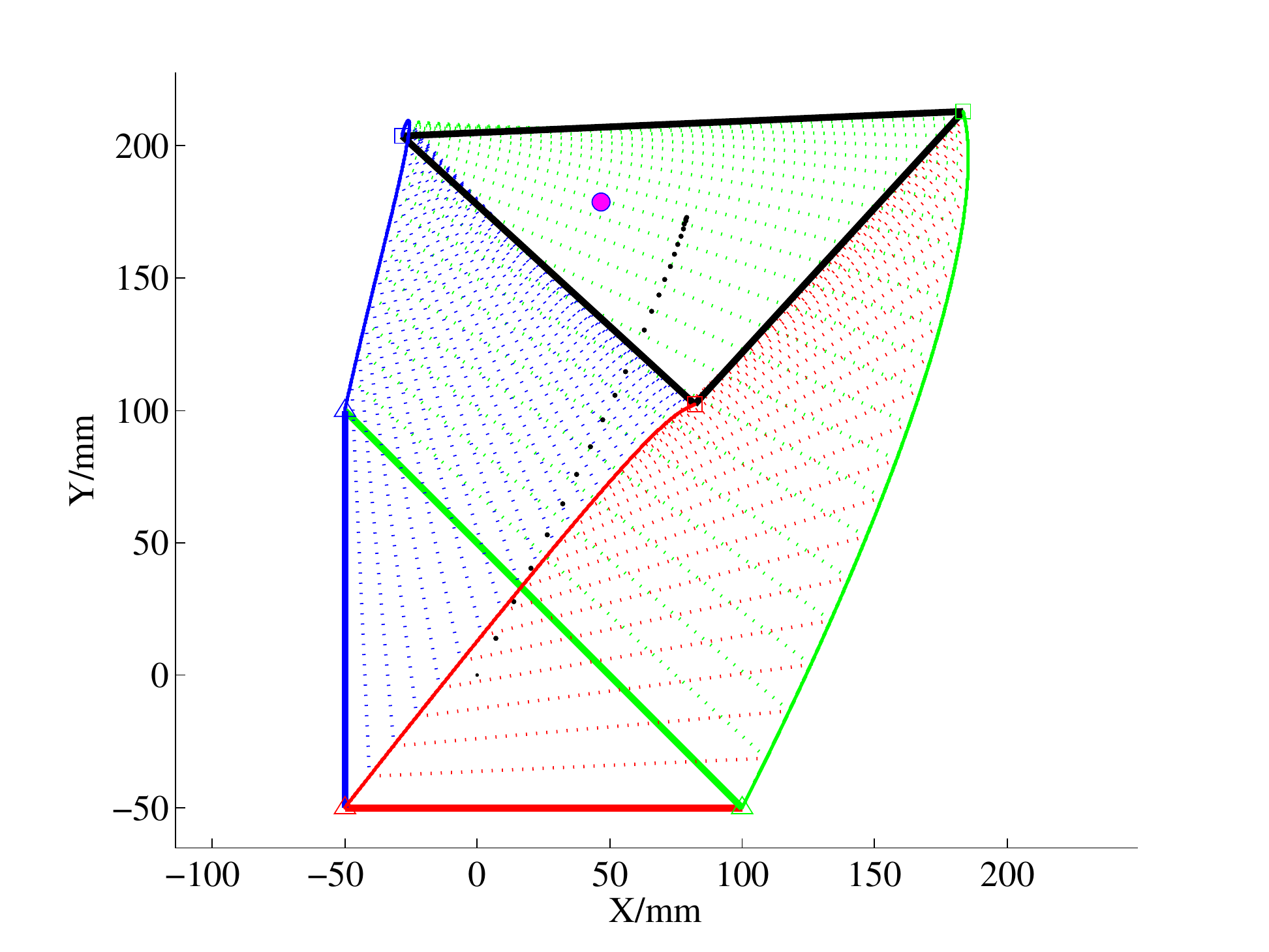}
\caption{\small Trajectory with $\mathbf{V}(0)$ equals (500mm/s, 100mm/s, $\pi$rad/s).}
\label{fig:dyn3}
\end{subfigure} 
~
\begin{subfigure}{\mylen}
\centering
\includegraphics[width=1.6in]{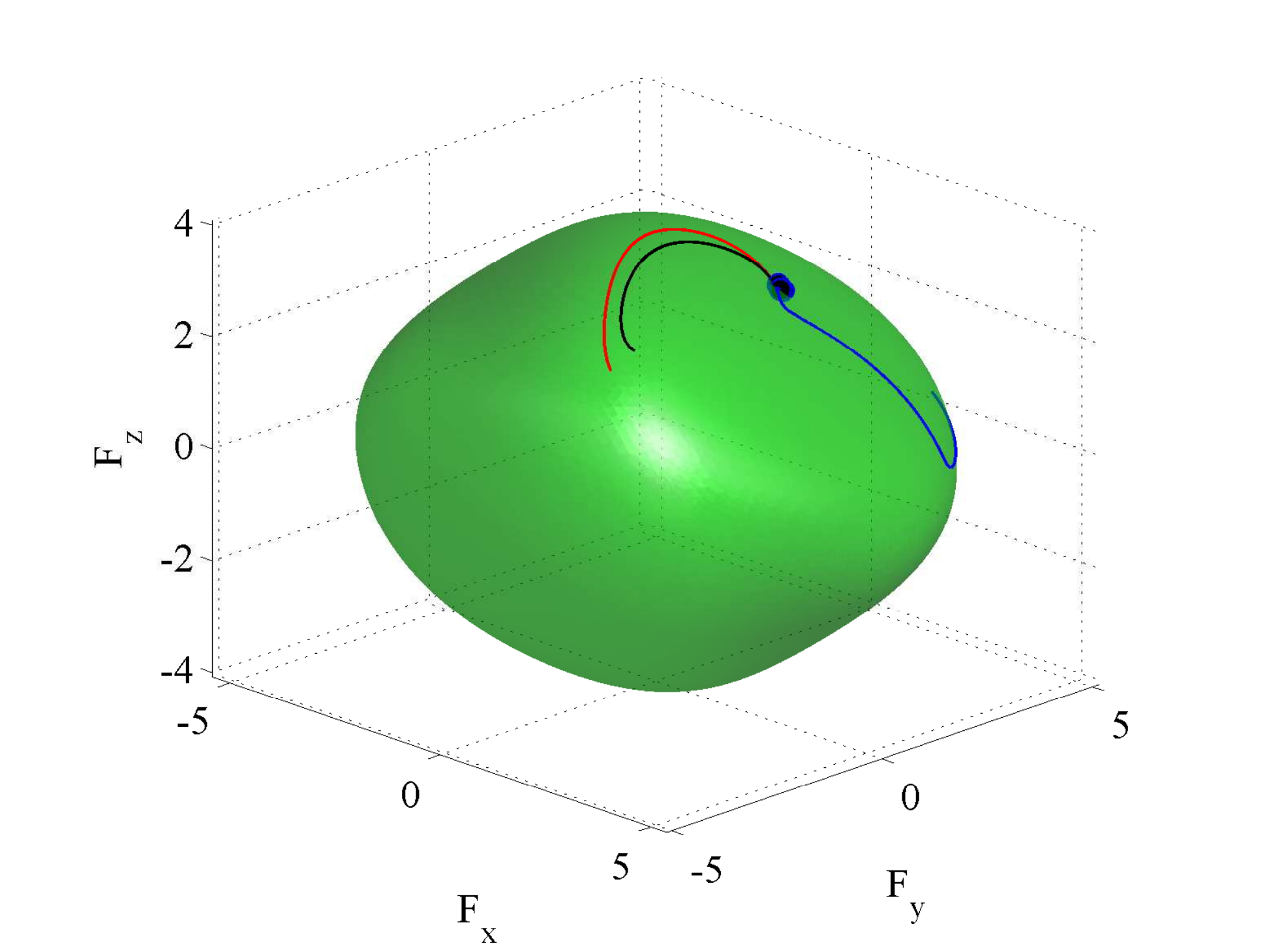}
\caption{\small Trajectories of generalized friction loads.}
\label{fig:dynlc}
\end{subfigure} 

\caption{\small Dynamic simulations based on a polynomial model using 10 training and 10 validation data for a triangular object with three support points as in Fig. \ref{fig:robot_lc_tri}. 
Fig. \ref{fig:dyn1}, \ref{fig:dyn2}, and \ref{fig:dyn3} illustrate the
sliding trajectories for different initial velocities. 
The rgb and black triangles are the initial and final poses respectively, with dotted triangles
as intermediate poses. Black dots traces correspond to
the trajectories of center of mass. Magenta circles are the
instantaneous rotation centers at the final time steps. 
Trajectories of the generalized friction loads are shown in Fig. \ref{fig:dynlc},
where red, black, and blue curves correspond to Fig. \ref{fig:dyn1}, \ref{fig:dyn2}, and
\ref{fig:dyn3} respectively.}
\label{fig:dyn}
\vspace{-0.15in}
\end{figure}

\section{CONCLUSION AND FUTURE WORK}
In this paper, we propose to use the sub-level sets and gradients of a
function to represent rigid body planar friction loads and velocities,
respectively.
The maximum work inequality implies that such a function needs to be
convex. We additionally require the properties of symmetry, scale
invariance, and efficient invertibility which lead us to choose 
a convex even-degree homogeneous polynomial representation. 
We apply the representation to applications including stable pushing and dynamic simulation. 
For future work, we plan to evaluate the model on a larger dataset with varying object and surface material physical properties.
We will also explore methods for online model
identification. 


\section*{ACKNOWLEDGMENT}
This work was conducted in part through collaborative participation in
the Robotics Consortium sponsored by the U.S Army Research Laboratory
under the Collaborative Technology Alliance Program, Cooperative
Agreement W911NF-10-2-0016 and National Science Foundation IIS-1409003. The views and conclusions contained in this document are those of the authors and should not be interpreted as representing the official policies, either expressed or implied, of the Army Research Laboratory of the U.S. Government. The U.S. Government is authorized to reproduce and distribute reprints for Government purposes notwithstanding any copyright notation herein.
\bibliographystyle{ieeetr}
{\footnotesize
\bibliography{ref}}%

\end{document}